\documentclass[table,dvipsnames]{article} 

\usepackage{algorithm} 
\usepackage{algorithmic} 
\usepackage{wrapfig,lipsum,booktabs}
\newcommand{\veryshortarrow}[1][3pt]{\mathrel{%
   \hbox{\rule[\dimexpr\fontdimen22\textfont2-.2pt\relax]{#1}{.4pt}}%
   \mkern-4mu\hbox{\usefont{U}{lasy}{m}{n}\symbol{41}}}} 
\newcommand{\cmark}{\ding{51}}
\newcommand{\xmark}{\ding{55}}
\usepackage{amssymb}
\usepackage{pifont}
\usepackage[hyphens]{url}
\usepackage{amsthm} 
\usepackage{graphicx} 
\newtheorem{theorem}{Theorem} 
\newtheorem{lemma}[theorem]{Lemma} 
\newtheorem{definition}{Definition}[section] 
\newtheorem{conjecture}{Conjecture} 
\usepackage[dvipsnames]{xcolor} 
\definecolor{mygray}{gray}{0.9} 

\usepackage{multirow} 
\usepackage{adjustbox} 

\newcommand\Tstrut{\rule{0pt}{2.6ex}}         
\newcommand\Bstrut{\rule[-0.9ex]{0pt}{0pt}}   
\usepackage{caption} 
\definecolor{coolblack}{rgb}{0.0, 0.23, 0.64} 
\usepackage[pagebackref=true,breaklinks=true,letterpaper=true,colorlinks,bookmarks=false]{hyperref} 
\hypersetup{
     citecolor = coolblack
} 
\usepackage{caption}
\usepackage{comment} 
\usepackage{listings} 
\usepackage[compact]{titlesec}
\titlespacing{\section}{0pt}{1ex}{0.5ex}
\usepackage{placeins}
\usepackage{afterpage}
\usepackage{minitoc}

\usepackage[accepted]{icml2022}

\usepackage{amsmath}
\usepackage{amssymb}
\usepackage{mathtools}
\usepackage{amsthm}

\definecolor{coolblack}{rgb}{0.0, 0.23, 0.64} 

\usepackage{caption}
\usepackage{comment} 
\usepackage{subcaption}
\usepackage{titletoc}
\usepackage[toc,page,header]{appendix}
\usepackage[capitalize,noabbrev]{cleveref}

\theoremstyle{plain}

\theoremstyle{definition}

\theoremstyle{remark}

\usepackage{multirow} 
\usepackage{adjustbox} 
\definecolor{mygray}{gray}{0.9} 

\usepackage{caption} 
\usepackage[normalem]{ulem}
\usepackage[textsize=tiny]{todonotes}

\icmltitlerunning{A Closer Look at Smoothness in Domain Adversarial Training}
\begin{document}
\definecolor{codegreen}{rgb}{0,0.6,0}
\definecolor{codegray}{rgb}{0.5,0.5,0.5}
\definecolor{codepurple}{rgb}{0, 0, 0}
\definecolor{backcolour}{rgb}{0.95,0.95,0.92}

\lstdefinestyle{mystyle}{
    backgroundcolor=\color{backcolour},   
    commentstyle=\color{magenta},
    keywordstyle=\color{magenta},
    numberstyle=\tiny\color{codegray},
    stringstyle=\color{codepurple},
    basicstyle=\ttfamily\footnotesize,
    breakatwhitespace=false,         
    breaklines=true,                 
    captionpos=b,                    
    keepspaces=true,                 
    numbers=left,                    
    numbersep=1pt,                  
    showspaces=false,                
    showstringspaces=false,
    showtabs=false,                  
    tabsize=2,
    language=Python,
}

\lstset{style=mystyle}
\twocolumn[
\icmltitle{ A Closer Look at Smoothness in Domain Adversarial Training}



\icmlsetsymbol{equal}{*}

\begin{icmlauthorlist}
\icmlauthor{Harsh Rangwani}{equal,yyy}
\icmlauthor{Sumukh K Aithal}{equal,yyy,comp}
\icmlauthor{Mayank Mishra}{yyy}
\icmlauthor{Arihant Jain}{yyy,sch}
\icmlauthor{R. Venkatesh Babu}{yyy}
\end{icmlauthorlist}

\icmlaffiliation{yyy}{Video Analytics Lab, Indian Institute of Science, Bengaluru, India}
\icmlaffiliation{comp}{PES University, Bengaluru}
\icmlaffiliation{sch}{Amazon, India (Work done at Indian Institute of Science, Bengaluru)}

\icmlcorrespondingauthor{Harsh Rangwani}{harshr@iisc.ac.in}

\icmlkeywords{Machine Learning, ICML, Domain Adaptation}
\vskip 0.3in
]
\printAffiliationsAndNotice{\icmlEqualContribution}




\begin{abstract}
Domain adversarial training has been ubiquitous for achieving invariant representations and is used widely for various domain adaptation tasks. In recent times, methods converging to smooth optima have shown improved generalization for supervised learning tasks like classification.  In this work, we analyze the effect of smoothness enhancing formulations on domain adversarial training, the objective of which is a combination of {task loss (eg.\ classification, regression etc.)} and adversarial terms. We find that converging to a smooth minima with respect to (w.r.t.) task loss stabilizes the adversarial training leading to better performance on target domain. In contrast to {task} loss, our analysis shows that {converging to smooth minima w.r.t. adversarial loss leads to sub-optimal generalization on the target domain}. Based on the analysis, we introduce the Smooth Domain Adversarial Training (SDAT) procedure, which effectively enhances the performance of existing domain adversarial methods for both classification and object detection tasks.  Our analysis also provides insight into the extensive usage of SGD over Adam in the community for domain adversarial training. 
\end{abstract}
\begin{figure*}[!t]
  \centering
  \includegraphics[width=0.95\textwidth]{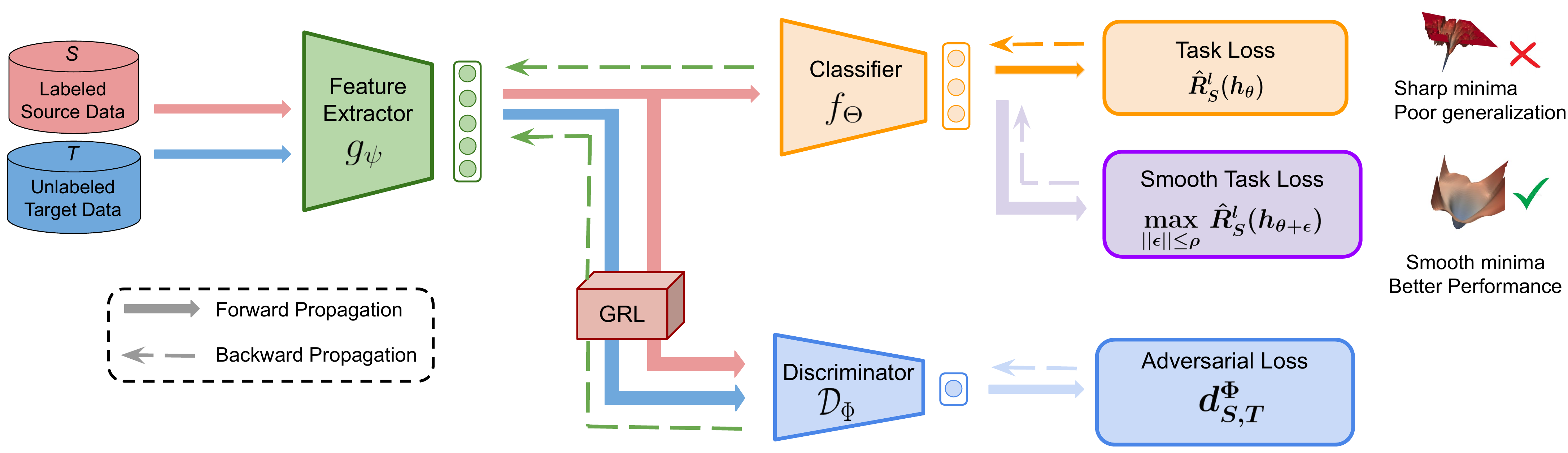}
  \caption{Overview of Smooth Domain Adversarial Training (SDAT). We demonstrate that converging to smooth minima w.r.t. adversarial loss leads to sub-optimal DAT. Due to this conventional approaches which smooth combination of task loss and adversarial loss lead to sub-optimal results. Hence, we propose SDAT which only focuses on smoothing {task} loss, leading to stable training which results in effective generalization on target domain.  \footnotemark}
  \label{fig:overview}
\end{figure*}
\section{Introduction}
    Domain Adversarial Training \citep{ganin2015unsupervised} (DAT) refers to adversarial learning of neural network based feature representations that are invariant to the domain. For example, the learned feature representations for car images from the Clipart domain should be similar to that from the Web domain. DAT has been widely useful in diverse areas (cited 4200 times) such as recognition \citep{long2018conditional,cui2020gvb,Rangwani_2021_ICCV}, fairness \citep{adel2019one}, object detection \citep{saito2019strong}, domain generalization \citep{li2018domain}, image-to-image translation \citep{liu2017unsupervised} etc. The prime driver of research on DAT is its application in unsupervised Domain Adaptation (DA), which aims to learn a classifier using labeled source data and unlabeled target data, such that it generalizes well on target data. Various enhancements like superior objectives \citep{acuna2021f, zhang2019bridging}, architectures \citep{long2018conditional}, etc. have been proposed to improve the effectiveness of DAT for unsupervised DA.

As DAT objective is a combination of Generative Adversarial Network (GAN) \citep{goodfellow2014generative} (adversarial loss)  and Empirical Risk Minimization (ERM) \citep{vapnik2013nature} (task loss) objectives, there has not been much focus on explicitly analyzing and improving the nature of optimization in DAT. In optimization literature, one direction that aims to improve the generalization focuses on developing algorithms that converge to a smooth (or a flat) minima \citep{foret2021sharpnessaware, keskar2017improving}. However, we find that these techniques, when directly applied for DAT, do not significantly improve the generalization on the target domain (Sec. \ref{sec:discussion}). \\ \\
In this work, we analyze the loss landscape near the optimal point obtained by DAT to gain insights into the nature of optimization. We first focus on the eigen-spectrum of Hessian (i.e.\ curvature) of the  {task loss (ERM term for classification)} where we find that using Stochastic Gradient Descent (SGD) as optimizer converges to a smoother minima in comparison to Adam \citep{kingma2014adam}. Further, we find that \textit{smoother minima w.r.t. {task} loss results in stable DAT leading to better generalization on the target domain}. Contrary to {task} loss, we find that smoothness enhancing formulation for adversarial components worsens the performance, rendering techniques \cite{cha2021swad} which enhance smoothness for all loss components ineffective. Hence we introduce Smooth Domain Adversarial Training (SDAT) (Fig. \ref{fig:overview}), which aims to reach a smooth minima only w.r.t. task loss and helps in generalizing better on the target domain. SDAT requires an additional gradient computation step and can be combined easily with existing methods. 
We show the soundness of the SDAT method theoretically through a generalization bound (Sec.\ \ref{smoothness}) on target error. We extensively verify the empirical efficacy of SDAT over DAT across various datasets for classification (i.e., DomainNet, VisDA-2017 and Office-Home) with ResNet and Vision Transformer \cite{dosovitskiy2020image} (ViT) backbones. We also show a prototypical application of SDAT in DA for object detection, demonstrating it's diverse applicability. In summary, we make the following contributions:
\begin{itemize}
    \item We demonstrate that converging to smooth minima w.r.t. task loss leads to stable and effective domain alignment through DAT, whereas smoothness enhancing formulation for adversarial loss leads to sub-optimal performance via DAT.
    \item For enhancing the smoothness w.r.t. task loss near optima in DAT, we propose a simple, novel, and theoretically motivated SDAT formulation that leads to stable DAT resulting in improved generalization on the target domain. 
    \item We find that SDAT, when combined with the existing state-of-the-art (SOTA) baseline for DAT, leads to significant gains in performance. Notably, with ViT backbone, SDAT leads to a significant effective average gain of \textbf{3.1\%} over baseline, producing SOTA DA performance without requiring any additional module (or pre-training data) using only a 12 GB GPU. The source code used for experiments is available at: \url{https://github.com/val-iisc/SDAT}.
\end{itemize}

\section{Related Work}
\footnotetext{Figures for the smooth minima and sharp minima are from \citep{foret2021sharpnessaware} and used for illustration purposes only.}
{ \textbf{Unsupervised Domain Adaptation}: It refers to a class of methods that enables the model to learn representations from the source domain's labeled data that generalizes well on the unseen data from the target domain \citep{long2018conditional, acuna2021f, zhang2019bridging, Kundu_2021_ICCV, Kundu_2020_CVPR}. One of the most prominent lines of work is based on DAT \citep{ganin2015unsupervised}. 
It involves using an additional discriminator to distinguish between source and target domain features. A Gradient Reversal layer (GRL) is introduced to achieve the goal of learning domain invariant features. The follow-up works have improved upon this basic idea by introducing a class information-based discriminator (CDAN \citep{long2018conditional}), introducing a transferable normalization function \citep{wang2019transferable}, using an improved Margin Disparate Discrepancy~\cite{zhang2019bridging} measure between source and target domain, etc. In this work, we focus on analyzing and improving such methods.}\\

{\textbf{Smoothness of Loss Landscape:} As neural networks operate in the regime of over parameterized models, low error on training data does not always lead to better generalization \citep{keskar2017large}. Often it has been stated \citep{hochreiter1997flat, hochreiter1994simplifying, he2019asymmetric, dziugaite2017computing}  that smoother minima does generalize better on unseen data. But until recently, this was practically expensive as smoothing required additional costly computations. Recently, a method called Sharpness Aware Minimization (SAM) \citep{foret2021sharpnessaware} for improved generalization has been proposed which finds a smoother minima with an additional gradient computation step. However, we find that just using SAM naively does not lead to improved generalization on target domain (empirical evidence in Tab. \ref{tab:diff_smooth},\ref{tab:officehome_erm} and \ref{table:visda_uda_vit_erm}). In this work, we aim to develop solutions which converge to a smooth minima but at same time lead to better generalization on target domain, which is not possible just by using SAM. }
\begin{figure*}[!t]
\centering
\begin{tabular}{c c c c}
\hspace{-0.4cm}\includegraphics[width=0.24\linewidth, height = 3.7cm]{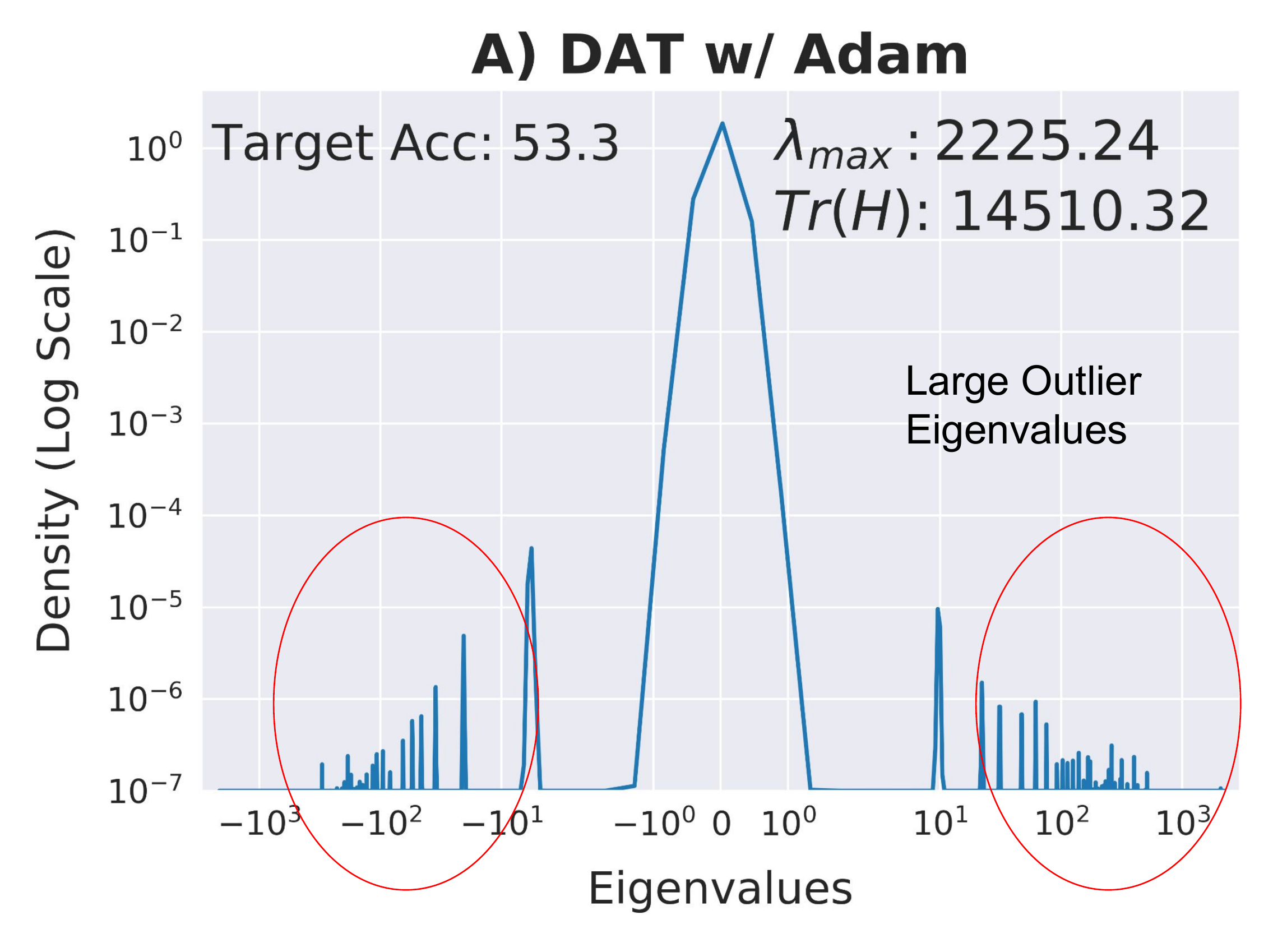} &   \hspace{-0.4cm}\includegraphics[width=0.24\linewidth, height = 3.7cm]{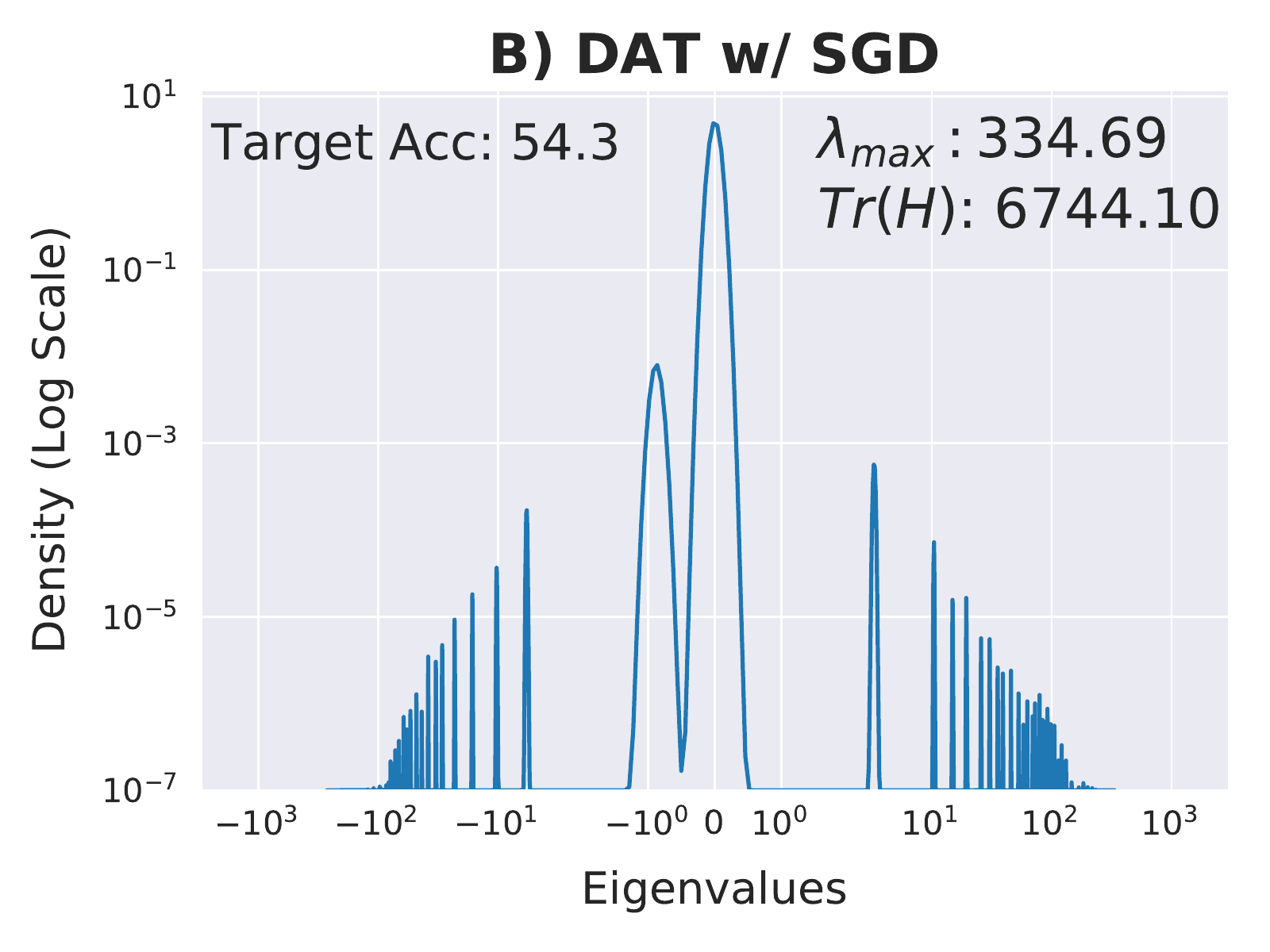} &
\hspace{-0.4cm}\includegraphics[width=0.24\linewidth, height = 3.7cm]{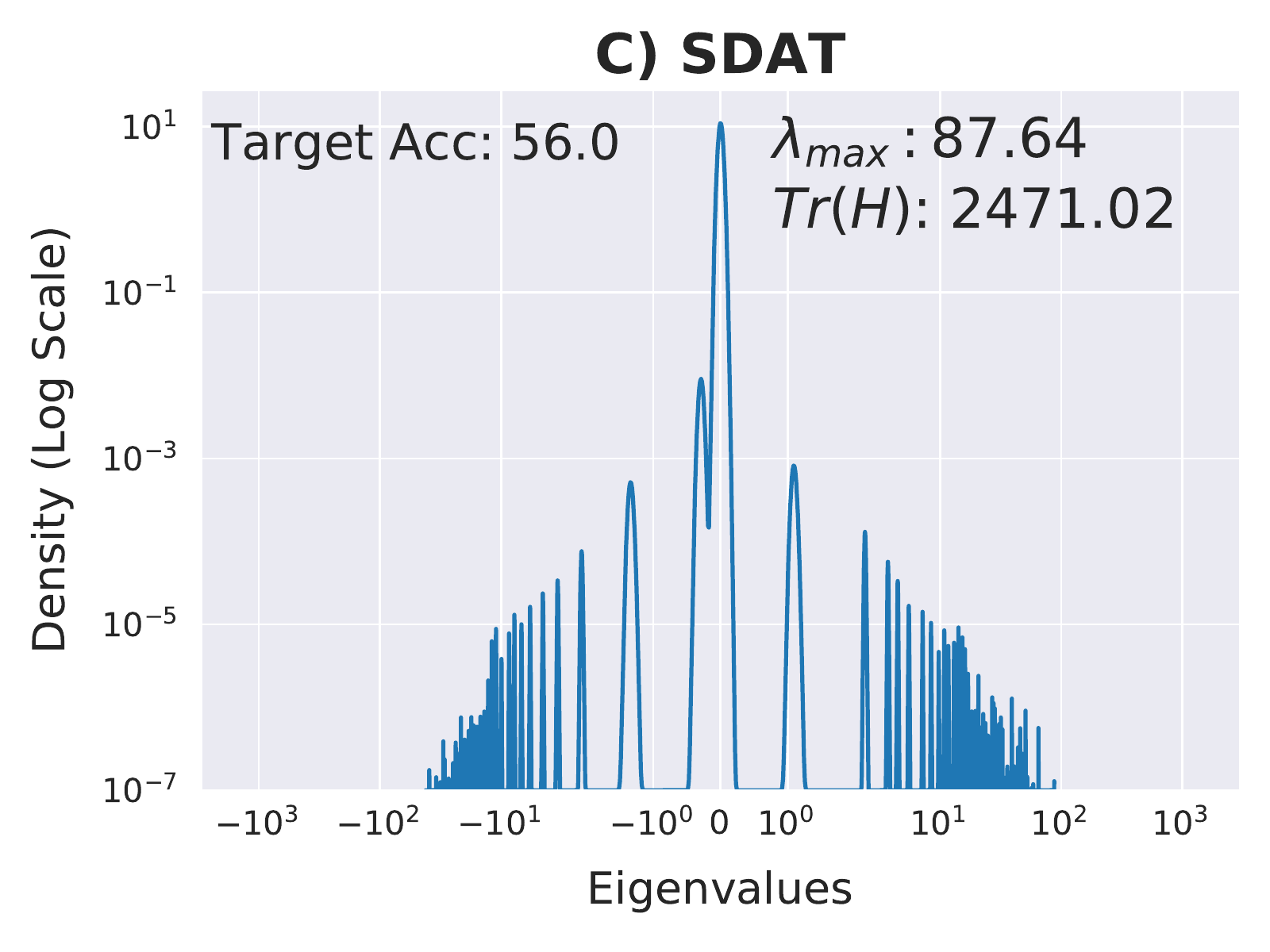} & \hspace{-0.4cm}\includegraphics[width=0.28\linewidth, height = 3.7cm]{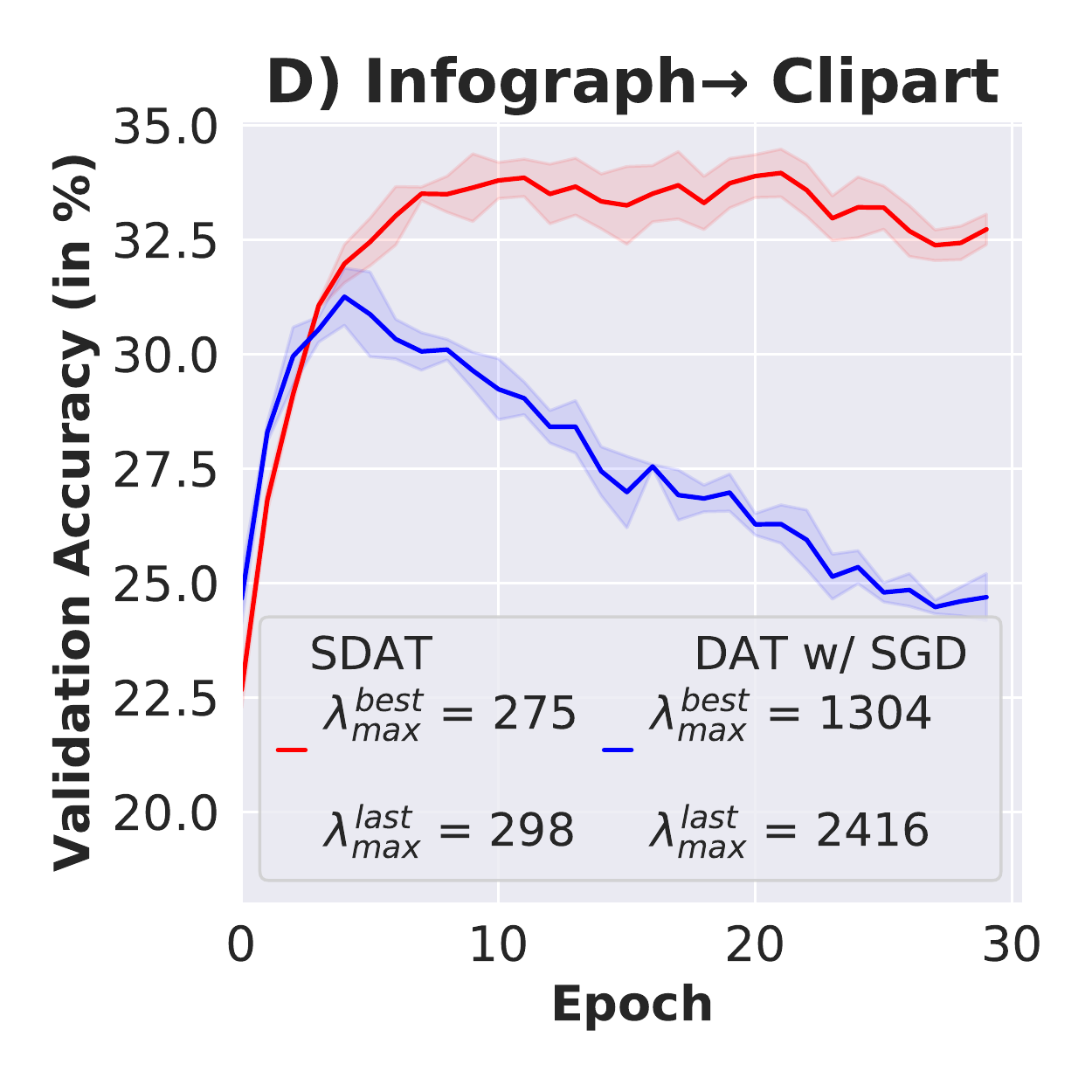} \\
\end{tabular}
\caption{Eigen Spectral Density plots of Hessian ($\nabla^2 \hat{R}_{S}^l(h_{\theta})$) for Adam (\textbf{A}), SGD (\textbf{B}) and SDAT (\textbf{C}) on Art $\veryshortarrow$ Clipart. Each plot contains the maximum eigenvalue ($\lambda_{max}$) and the trace of the Hessian ($Tr(H)$), which are indicators of the smoothness (Low $Tr(H)$ and $\lambda_{max}$ indicate the presence of smoother loss surface). Low range of eigenvalues (x-axis), $Tr(H)$ and $\lambda_{max}$ for SGD compared to Adam indicates that it reaches a smoother minima, which leads to a higher target accuracy. D) Validation accuracy and $\lambda_{\max}$ comparison for SDAT and DAT across epochs, SDAT shows significantly stable training with low $\lambda_{\max}$. }
\label{fig:hessian}
\end{figure*}
\section{Background}
\subsection{Preliminaries} 
We will primarily focus on unsupervised DA where we have labeled source data $S = \{ (x_i^s, y_i^s) \} $ and unlabeled target data $T = \{ (x_i^t) \}$. The source samples are assumed to be sampled i.i.d. from source distribution $P_S$ defined on input space $\mathcal{X}$, similarly target samples are sampled i.i.d. from $P_T$. $\mathcal{Y}$ is used for denoting the label set which is $\{1, 2,  \dots, k\}$ in our case as we perform multiclass ($k$) classification. We denote $y : \mathcal{X} \rightarrow \mathcal{Y}$ a mapping from images to labels. Our task is to find a hypothesis function $h_{\theta}$ that {has a low risk on the target distribution}. The source risk (a.k.a expected error) of the hypothesis $h_{\theta}$ is defined with respect to loss function $l$ as: $R_{S}^{l}(h_{\theta}) = \mathbb{E}_{x\sim P_S}[l(h_{\theta}(x), y(x))]$. The target risk $R_T^l(h_\theta)$ is defined analogously. The empirical versions of source and target risk will be denoted by $\hat{R}_S^l(h_\theta)$ and $\hat{R}_T^l(h_\theta)$. All notations used in paper are summarized in App. \ref{app:notn_tab}. In this work we build on the DA theory of \citep{acuna2021f} {which is a generalization of~\citet{ben2010theory}}. 
We first define the discrepancy between the two domains.

\begin{definition}[$ D_{h_{\theta},\mathcal{H}}^\phi$ discrepancy] The discrepancy between two domains $P_S$ and $P_T$ is defined as following:
\begin{equation}
    \begin{split}
    D_{h_{\theta}, \mathcal{H}}^{\phi}(P_S || P_T) := \sup_{h' \in \mathcal{H}} [\mathbb{E}_{x \sim P_S}[l(h_{\theta}(x),h'(x))]] - \\ [\mathbb{E}_{x \sim P_T}[\phi^*(l(h_{\theta}(x),h'(x)))]]
\end{split}
\end{equation}
Here $\phi^*$ is a frenchel conjugate of a lower semi-continuous convex function $\phi$  that satisfies $\phi(1) = 0$, and $\mathcal{H}$ is the set of all possible hypothesis (i.e. Hypothesis Space).
\end{definition}

This discrepancy distance $D_{h_{\theta}, \mathcal{H}}^\phi$ is based on variational formulation of f-divergence \citep{nguyen2010estimating} for the convex function $\phi$. The $D_{h_{\theta}, \mathcal{H}}^{\phi}$ is the lower bound estimate of the f-divergence function $D^{\phi}(P_S||P_T)$ (Lemma 4 in \citep{acuna2021f}). We state a bound on target risk $R_{T}^l(h_{\theta})$ based on $\mathcal D_{h_{\theta},\mathcal{H}}^\phi$ discrepancy \citep{acuna2021f}:
\begin{theorem}[\textbf{Generalization bound}]
\label{th:gen-bound}
Suppose $l: \mathcal{Y} \times \mathcal{Y} \rightarrow [0,1] \subset dom \; \phi^*$. Let $h^*$ be the ideal joint classifier with {least} $\lambda^* = R_S^l(h^*) +  R_T^l(h^*)$ {(i.e. joint risk)} in $\mathcal{H}$. We have the following relation between source and target risk:
\begin{equation}
    R_{T}^l(h_{\theta}) \leq R_{S}^{l}(h_{\theta}) + D_{h_{\theta}, \mathcal{H}}^{\phi} (P_S || P_T) + \lambda^*
\end{equation}
\end{theorem}
The above generalization bound shows that the target risk $R_T^l(h_{\theta})$ is upper bounded by the source risk $R_S^l(h_{\theta})$ and the discrepancy term $D_{h_{\theta}, \mathcal{H}}^\phi$ along with an irreducible constant error $\lambda^*$. Hence, this infers that reducing source risk and discrepancy lead a to reduction in target risk. Based on this, we concretely define the unsupervised adversarial adaptation procedure in the next section.

\subsection{Unsupervised Domain Adaptation}
In this section we first define the components of the framework we use for our purpose: $h_{\theta} = f_{\Theta} \circ g_{\psi}$ where $g_{\psi}$ is the feature extractor  and $f_{\Theta}$ is the classifier. The domain discriminator $\mathcal{D}_{\Phi}$, used for estimating the discrepancy between $P_S$ and $P_T$ is a classifier whose goal is to distinguish between the features of two domains. For minimizing the target risk (Th. \ref{th:gen-bound}), the optimization problem is as follows:
\begin{equation}
    \underset{\theta}{\min} \; \mathbb{E}_{x \sim P_S}[l(h_{\theta}(x), y(x))] + D_{h_{\theta}, \mathcal{H}}^{\phi}(P_S||P_T)
    \label{eq:uda_obj_min}
\end{equation}
The discrepancy term under some assumptions (refer App. \ref{app:discrepancy}) can be upper bounded by a tractable term:
\begin{equation}
\begin{split}
        D_{h_{\theta}, \mathcal{H}}^{\phi}(P_S||P_T) \leq \underset{\Phi}{\max} \; d_{S,T}^{\Phi}
\end{split}
\label{eq:diver_disc}
\end{equation}
where 
$d_{S,T}^{\Phi} = \mathbb{E}_{x \sim P_S}[\log(\mathcal{D}_{\Phi}(g_{\psi}(x)))] + \mathbb{E}_{x \sim P_T}\log[1 - \mathcal{D}_{\Phi}(g_{\psi}(x))]$.
This leads to the final optimization objective of:
\begin{equation}
        \underset{\theta}{\min} \; \underset{\Phi}{\max} \; \mathbb{E}_{x \sim P_S}[l(h_{\theta}(x), y(x))] +  \; d_{S,T}^{\Phi}
\label{eq:uda_obj}
\end{equation}
The first term in practice is empirically approximated by using finite samples $\hat{R}_S^l(h_\theta)$ and used as {task} loss (classification) for minimization. The empirical estimate of the second term is adversarial loss which is optimized using GRL as it has a min-max form. (Overview in Fig. \ref{fig:overview}) The above procedure composes DAT, and we use CDAN \citep{long2018conditional} as our default DAT method.

\section{Analysis of Smoothness} \label{smoothness}
In this section, we analyze the curvature properties of the task loss with respect to the parameters ($\theta$). Specifically, we focus on analyzing the Hessian of empirical source risk $H = \nabla^2_{\theta} \hat{R}_S^l(h_{\theta})$ which is the Hessian of classification ({task}) loss term. For quantifying the smoothness, we measure the trace $Tr(H)$ and maximum eigenvalue of Hessian ($\lambda_{max}$) as a proxy for quantifying smoothness. 
This is motivated by analysis of 
which states that the low value of $\lambda_{max}$ and $Tr(H)$ are indicative of highly smooth loss landscape \citep{Jastrzebski2020The}. Based on our observations we articulate our conjecture below: 
\begin{conjecture}
  Low $\lambda_{\max}$ for Hessian  of empirical source risk (i.e. task loss) $\nabla^2_{\theta}\hat{R}_S^l(h_{\theta})$ leads to stable and effective DAT, resulting in reduced risk on target domain  $\hat{R}^l_{T}(h_{\theta})$.
  
\end{conjecture}

\begin{figure*}[t]
\centering
\begin{tabular}{c c c}
\includegraphics[width=0.3\linewidth]{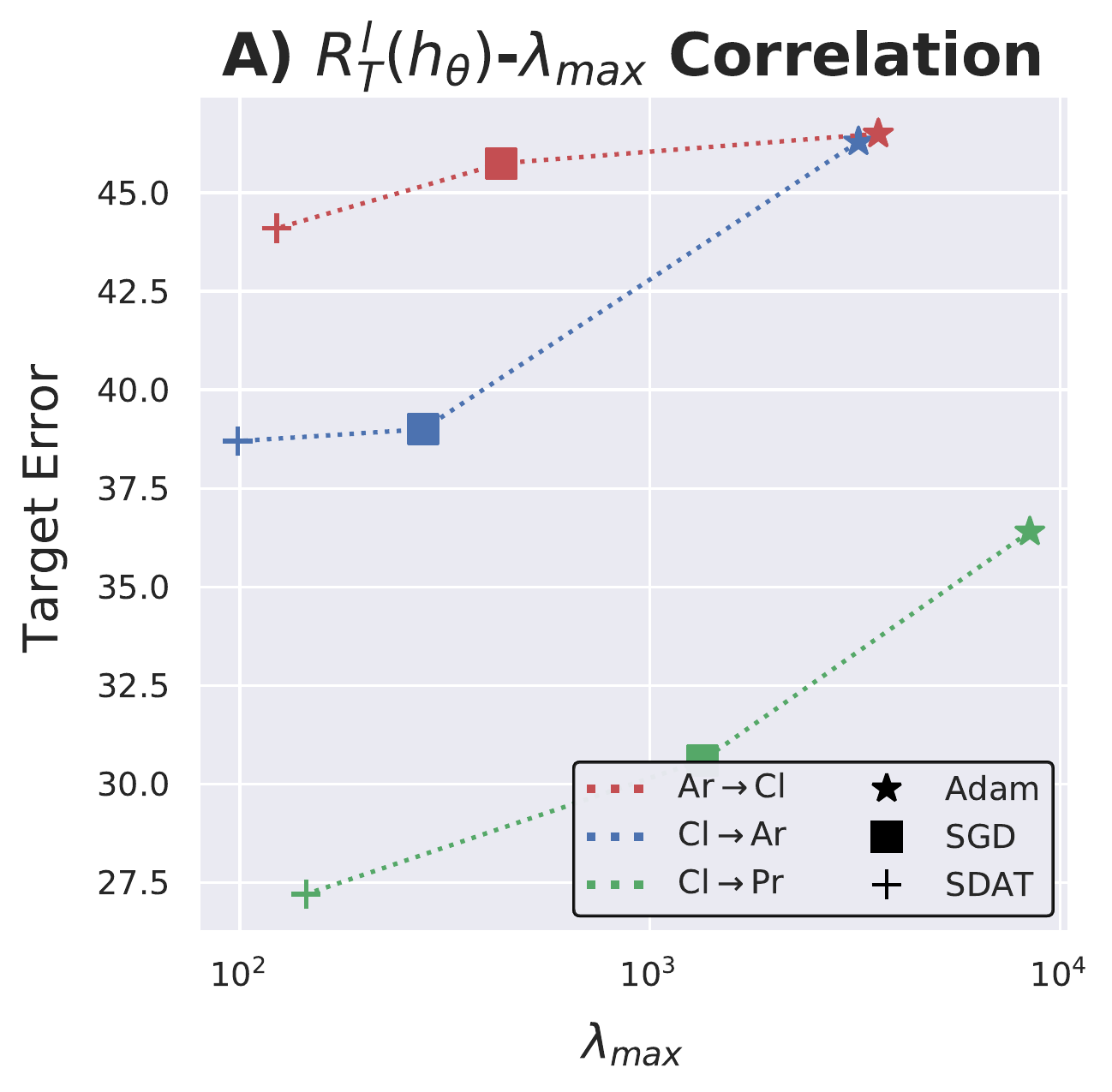} &   \includegraphics[width=0.3\linewidth]{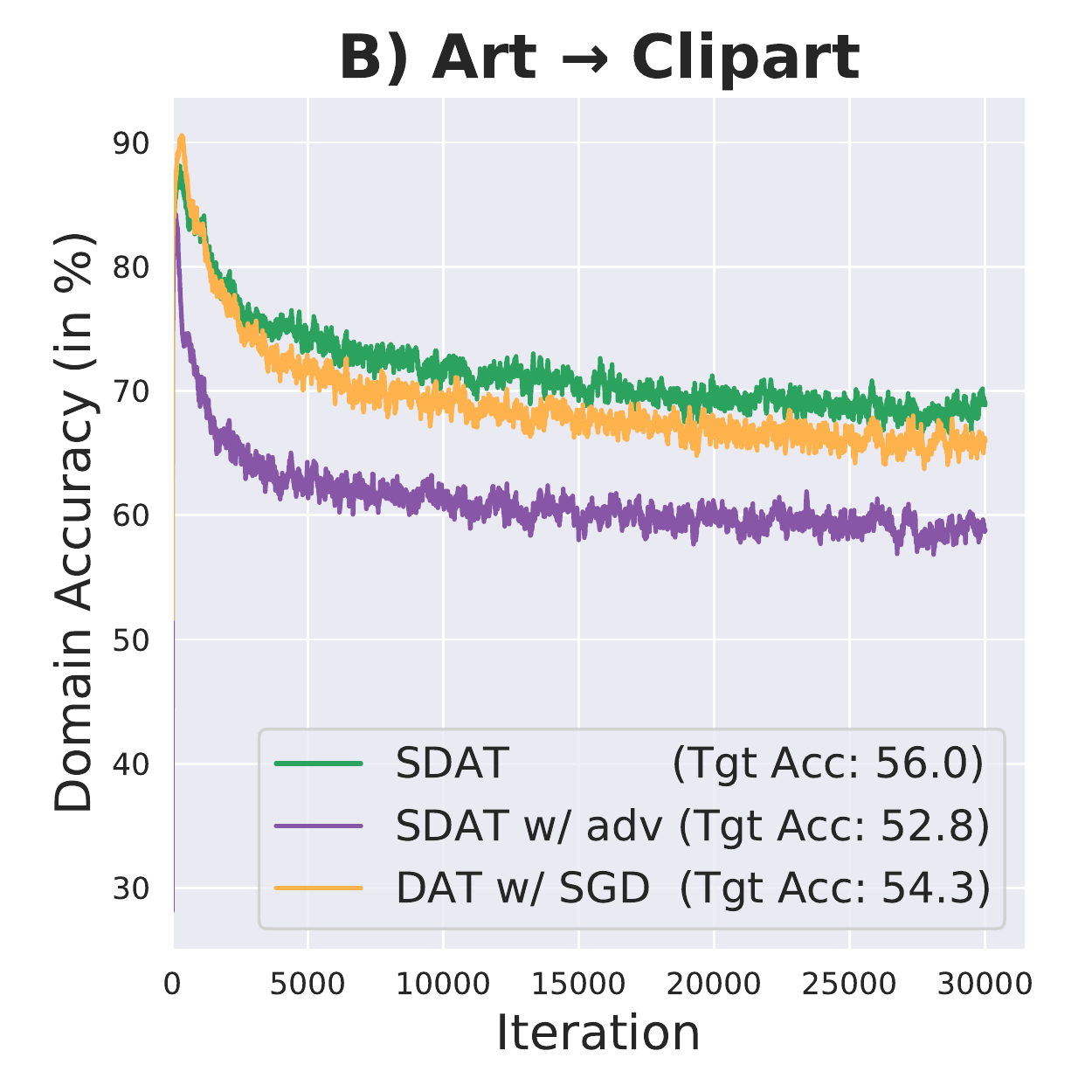} & \includegraphics[width=0.3\linewidth]{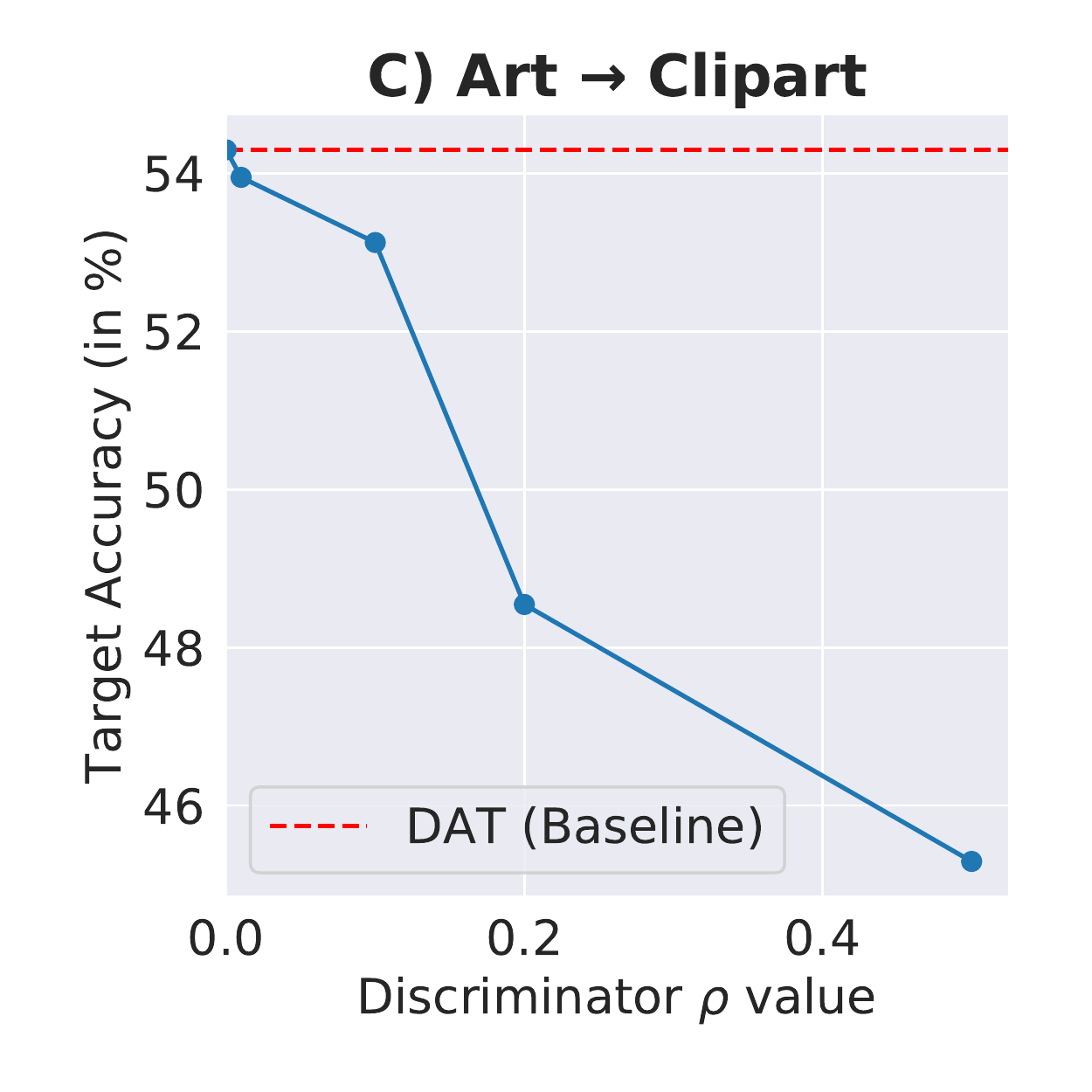} \\
\end{tabular}
\caption{\textbf{A)} Error on Target Domain (y-axis) for Office-Home dataset against maximum eigenvalue $\lambda_{max}$ of classification loss in DAT. 
When compared to SGD, Adam converges to a non-smooth minima (high $\lambda_{max}$), leading to a high error on target. 
Using Adam in comparison to SGD, converges to a non-smooth minima (high $\lambda_{max}$) leading to high error on target.
\textbf{B)} Domain Accuracy (vs iterations), it is lower when adversarial loss is smooth (i.e. SDAT w/ adv), which indicates suboptimal discrepancy estimation $d_{s,t}^{\Phi}$ \textbf{C)} Target Accuracy on Art $\rightarrow$ Clipart vs smoothness of the adversarial component. As the smoothness increases ($\rho$), the target accuracy decreases indicating that smoothing adversarial loss leads to sub-optimal generalization.}
\label{fig:2}
\end{figure*}

For empirical verification of our conjecture, we obtain the Eigen Spectral Density plot for the Hessian $\hat{R}^l_{T}(h_{\theta})$. We show the $\lambda_{max}$, $Tr(H)$ and Eigen Spectrum for different algorithms, namely DAT w/ Adam, DAT w/ SGD and our proposed SDAT (which is described in detail in later sections) in Fig.\ \ref{fig:hessian}. We find that \textit{high smoothness leads to better generalization on the target domain} (Additional empirical evidence in Fig.\ \ref{fig:2}A). We hypothesize that enforcing smoothness of classifier $h_{\theta}$ leads to a smooth landscape for discrepancy ($d_{S,T}^{\Phi}$)  as it is also a function of $h_{\theta}$. The smooth landscape ensures stable minimization (of Eq. \ref{eq:uda_obj}), ensuring a decrease in ($d_{S,T}^{\Phi}$)  with each SGD step even for a large step size (similar to \citep{chu2020smoothness}), this explains the enhanced stability and improved performance of adversarial training. For verifying the stabilization effect of smoothness, empirically we obtain $\lambda_{\max}$ for SGD and proposed SDAT at both best ($\lambda_{\max}^{best}$) and last epoch ($\lambda_{\max}^{last}$) for adaptation from Infographic to Clipart Domain (Fig.\ \ref{fig:hessian}\textcolor{mydarkblue}{D}). We find that as $\lambda_{\max}$ increases (decrease in smoothness of landscape), the training becomes unstable for SGD leading to a drop in validation accuracy. Whereas in the case of the proposed SDAT, the $\lambda_{\max}$ remains low across epochs, leading to stable and better validation accuracy curve. We also provide additional validation accuracy curves for more adaptation tasks where we also observe a similar phenomenon in Fig. \ref{fig:plotss}. To the best of our knowledge, our analysis of the effect of smoothness of task loss on the stability of DAT is novel. \\
\\
We also find that SGD leads to low $\lambda_{\max}$ (high smoothness w.r.t. task loss) in comparison to Adam leading to better performance. This also explains the widespread usage of SGD for DAT algorithms \citep{ganin2015unsupervised, long2018conditional, saito2018adversarial}, instead of Adam. More details about Hessian analysis is provided in App. \ref{app:hess}.

\subsection{Smoothing Loss Landscape}
\label{sec:subopt_dis}
In this section we first introduce the losses which are based on Sharpness Aware Minimization \citep{foret2021sharpnessaware} (SAM). The basic idea of SAM is to find a smoother minima (i.e. low loss in $\epsilon$ neighborhood of $\theta$) by using the following objective given formally below:
\begin{equation}
    \underset{\theta}{\min} \;\underset{||\epsilon|| \leq \rho}{\max}\; L_{obj}(\theta + \epsilon) 
\end{equation}
here $L_{obj}$ is the objective function to be minimized and $\rho\geq0$ is a hyperparameter which defines the maximum norm for $\epsilon$. Since finding the exact solution of inner maximization is hard, SAM maximizes the first order approximation:
\begin{equation}
\begin{split}
            \hat{\epsilon}(\theta) \approx \underset{||\epsilon|| \leq \rho}{\arg \max} \;  L_{obj}(\theta) + \epsilon^{\mathbf{T}}\nabla_{\theta} L_{obj}(\theta) \\
            = \rho \nabla_{\theta} L_{obj}(\theta) / ||\nabla_{\theta} L_{obj}(\theta)||_2
\end{split}
\end{equation}

The $\hat{\epsilon}(\theta)$ is added to the weights $\theta$. The gradient update for $\theta$ is then computed as $\nabla_{\theta} L_{obj}(\theta)|_{\theta + \hat{\epsilon}(\theta)}$. The above procedure can be seen as a generic smoothness enhancing formulation for any $L_{obj}$. We now analogously introduce the sharpness aware source risk for finding a smooth minima:
\begin{equation}
      \underset{||\epsilon|| \leq \rho}{\max} \; {R}_{S}^l(h_{\theta + \epsilon}) =  
      \underset{||\epsilon|| \leq \rho}{\max} \mathbb{E}_{x \sim P_S}[\; l(h_{\theta + \epsilon}(x), f(x))]
\end{equation}
We also now define the sharpness aware discrepancy estimation objective below:
\begin{equation}
    \max_{\Phi} \min_{||\epsilon|| \leq \rho} d_{S,T}^{\Phi + \epsilon}
    \label{eq:smooth_disc}
\end{equation}
As $d_{S,T}^{\Phi}$ is to be maximized the sharpness aware objective will have $\underset{||\epsilon|| \leq \rho}{\min}$ instead of $\underset{||\epsilon|| \leq \rho}{\max}$, as it needs to find smoother maxima. We now theoretically analyze the difference in discrepancy estimation for smooth version $d_{S,T}^{\Phi''}$ (Eq. \ref{eq:smooth_disc}) in comparison to non-smooth version $d_{S,T}^{\Phi'}$ (Eq. \ref{eq:diver_disc}). Assuming $\mathcal{D}_{\Phi}$ is a $L$-smooth function {(common assumption for non-convex optimization~\citep{carmon2020lower})}, 
$\eta$ is a small constant and $d_{S,T}^*$ the optimal discrepancy, the theorem states:

\begin{theorem}
\label{th:suboptimality}
For a given classifier $h_{\theta}$ and one step of (steepest) gradient ascent i.e. $\Phi' = \Phi + \eta (\nabla d_{S,T}^{\Phi}/||\nabla d_{S,T}^{\Phi}||)$ and $\Phi'' = \Phi + \eta (\nabla d_{S,T}^{\Phi}|_{\Phi + \hat{\epsilon}(\Phi)}/||\nabla d_{S,T}^{\Phi}|_{\Phi + \hat{\epsilon}(\Phi)}||)$ 
\begin{equation}
\begin{split}
         d_{S,T}^{\Phi'} - d_{S,T}^{\Phi''} \leq  \eta(1 - \cos \alpha)\sqrt{2L(d^*_{S,T} - d^{\Phi}_{S,T}) }
\end{split}
\end{equation}
where $\alpha$ is the angle between $\nabla d_{S,T}^{\Phi}$ and $\nabla d_{S,T}^{\Phi}|_{\Phi + \hat{\epsilon}(\Phi)}$. 
\end{theorem}

The $d_{S,T}^{\Phi'}$ (non-smooth version) can exceed $d_{S,T}^{\Phi''}$ (smooth discrepancy) significantly, as the term $d_{S,T}^* -  d_{S,T}^{\Phi}  \not\to 0$, as the $h_{\theta}$ objective is to oppose the convergence of $d_{S,T}^{\Phi}$ to optima $d_{S,T}^*$ (min-max training in Eq. \ref{eq:uda_obj}). Thus $d_{S,T}^{\Phi'}$ can be a better estimate of discrepancy in comparison to $d_{S,T}^{\Phi''}$. A better estimate of $d_{s,t}^{\Phi}$ helps in effectively reducing the discrepancy between $P_S$ and $P_T$, hence leads to reduced $R_{T}^l(h_{\theta})$. This is also observed in practice that smoothing the discriminator's adversarial loss (SDAT w/ adv in Fig.\ \ref{fig:2}\textcolor{mydarkblue}{B}) leads to low domain classification accuracy (proxy measure for $d_{s,t}^{\Phi}$) in comparison to DAT. 
Due to ineffective discrepancy estimation, SDAT w/ adv results in sub-optimal generalization on target domain i.e. high target error $R_{T}^l(h_{\theta})$ (Fig.\ \ref{fig:2}\textcolor{mydarkblue}{B}). We also observe that further increasing the smoothness of the discriminator w.r.t. adversarial loss (increasing $\rho$) leads to lowering of performance on the target domain (Fig.\ \ref{fig:2}\textcolor{mydarkblue}{C}). A similar trend is observed in GANs (App.\ \ref{app:gan_exp}) which also has a similar min-max objective. The proof of the above theorem and additional experimental details is provided in App. \ref{app:proof}.

\subsection{Smooth Domain Adversarial Training (SDAT)}
We propose smooth domain adversarial training which only focuses on converging to smooth minima w.r.t. {task} loss (i.e. empirical source risk), whereas preserves the original discrepancy term. We define the optimization objective of proposed Smooth Domain Adversarial Training below:

\vspace{-2em}
\begin{equation}\min_{\theta} \max_{\Phi} \underset{||\epsilon|| \leq \rho}{\max}  \mathbb{E}_{x \sim P_S}[ l(h_{\theta + \epsilon}(x), y(x))] + d_{S,T}^{\Phi}
    \label{eq:uda_obj_smooth}
\end{equation}
The first term is the sharpness aware risk, and the second term is the discrepancy term which is not smooth in our procedure. The term $d_{S,T}^{\Phi}$ estimates $D_{h_{\theta}, H}^{\phi}(P_S||P_T)$ discrepancy. We now show that optimizing Eq.\ \ref{eq:uda_obj_smooth} reduces $R_T^l(h_\theta)$ through a generalization bound. {This bound establishes that our proposed SDAT procedure is also consistent (i.e. in case of infinite data the upper bound is tight) similar to the original DAT objective (Eq. \ref{eq:uda_obj})}.  

\begin{theorem}

Suppose l is the loss function, we denote $\lambda^* := R_S^l(h^*) + R_T^l(h^*)$ and let $h^*$ be the ideal joint hypothesis:
\begin{equation}
\begin{split}
         R_{T}^l(h_{\theta}) \leq \; \max_{||\epsilon|| \leq \rho}\hat{R}_S^l(h_{\theta + \epsilon}) + D_{h_{\theta}, H}^{\phi}(P_S||P_T)   + \\ \gamma(||\theta||_2^2/\rho^2) + \lambda^* .
\end{split}
\end{equation}
where $\gamma: \mathbb{R}^{+} \rightarrow \mathbb{R}^{+}$ is a strictly increasing function.
\end{theorem}
\begin{table*}[!t]
  \centering     
      \caption{Accuracy (\%) on {Office-Home} for unsupervised DA (with ResNet-50 and ViT backbone). CDAN+MCC w/ SDAT outperforms other SOTA DA techniques. CDAN w/ SDAT improves over CDAN by 1.1\% with ResNet-50 and 3.1\% with ViT backbone. }  
    \vskip 0.15in
  \resizebox{\textwidth}{!}{%
  \begin{tabular}{l|c|cccccccccccc|c}
    \hline
    \textbf{Method} && \textbf{Ar$\veryshortarrow$Cl} & \textbf{Ar$\veryshortarrow$Pr} & \textbf{Ar$\veryshortarrow$Rw} & \textbf{Cl$\veryshortarrow$Ar} & \textbf{Cl$\veryshortarrow$Pr} & \textbf{Cl$\veryshortarrow$Rw} & \textbf{Pr$\veryshortarrow$Ar} & \textbf{Pr$\veryshortarrow$Cl} & \textbf{Pr$\veryshortarrow$Rw} & \textbf{Rw$\veryshortarrow$Ar} & \textbf{Rw$\veryshortarrow$Cl} & \textbf{Rw$\veryshortarrow$Pr} & \textbf{Avg} \Tstrut
    \Bstrut\\\hline \hline

    ResNet-50 \citep{he2016deep} &\parbox[t]{2mm}{\multirow{10}{*}{\rotatebox[origin=c]{90}{ResNet-50}}}& 34.9 & 50.0 & 58.0 & 37.4 & 41.9 & 46.2 & 38.5 & 31.2 & 60.4 & 53.9 & 41.2 & 59.9 & 46.1 \\
    DANN \citep{ganin2016domain} && 45.6 & 59.3 & 70.1 & 47.0 & 58.5 & 60.9 & 46.1 & 43.7 & 68.5 & 63.2 & 51.8 & 76.8 & 57.6 \\
    CDAN* \citep{long2018conditional} && 49.0 & {69.3} & 74.5 & 54.4 & {66.0} & {68.4} & {55.6} & 48.3 & 75.9 & 68.4 & 55.4 & {80.5} & 63.8 \\
    MDD \citep{zhang2019bridging} && 54.9 & 73.7 & 77.8 & 60.0 & 71.4 & 71.8 & 61.2 & 53.6 & 78.1 & 72.5 & 60.2 & 82.3 & 68.1 \\
    f-DAL \citep{acuna2021f}  && {56.7} & {\underline{77.0}} & {81.1} & {63.1} & {72.2} & {75.9} & {\underline{64.5}} & {54.4} & {81.0} & {72.3} & {58.4} & {83.7} & {70.0} \\ 
    
    SRDC \citep{tang2020unsupervised}  && {52.3} & {76.3} & {81.0} & {\textbf{69.5}} & {\underline{76.2}} & {\textbf{78.0}} & {\textbf{68.7}} & {53.8} & {\underline{81.7}} & {\textbf{76.3}} & {57.1} & {85.0} & {\underline{71.3}}\Bstrut\\\cline{1-1}\cline{3-15}
    CDAN && 54.3 & {70.6} & 76.8 & 61.3 & 69.5 & 71.3 & {61.7} & 55.3 & {80.5} & 74.8 & 60.1 & {84.2} & 68.4\\ 
        
        
    \cellcolor{mygray}CDAN w/ SDAT && \cellcolor{mygray}56.0 & \cellcolor{mygray}{72.2} & \cellcolor{mygray}{78.6} & \cellcolor{mygray}{62.5} & \cellcolor{mygray}{73.2} & \cellcolor{mygray}{71.8} & \cellcolor{mygray}{62.1} & \cellcolor{mygray}{55.9} & \cellcolor{mygray}80.3 & \cellcolor{mygray}\underline{75.0} & \cellcolor{mygray}61.4 & \cellcolor{mygray}{84.5} & \cellcolor{mygray}69.5\Bstrut\\\cline{1-1}\cline{3-15}
    {CDAN + MCC} && \underline{57.0} & {76.0} & \underline{81.6} & {64.9} & {75.9} & {75.4} & {63.7} & \underline{56.1} & 81.2 & {74.2} & \underline{63.9} & \underline{85.4} & \underline{71.3} \\
    \cellcolor{mygray}{CDAN + MCC w/ SDAT} && \cellcolor{mygray}\textbf{58.2} & \cellcolor{mygray}\textbf{77.1} & \cellcolor{mygray}\textbf{82.2} & \cellcolor{mygray}\underline{66.3} & \cellcolor{mygray}\textbf{77.6} & \cellcolor{mygray}\underline{76.8} & \cellcolor{mygray}{63.3} & \cellcolor{mygray}\textbf{57.0} & \cellcolor{mygray}\textbf{82.2} & \cellcolor{mygray}{74.9} & \cellcolor{mygray}\textbf{64.7} & \cellcolor{mygray}\textbf{86.0} & \cellcolor{mygray}\textbf{72.2}\Bstrut\\
    \hline \hline
	TVT \cite{yang2021tvt}  &\parbox[t]{2mm}{\multirow{5}{*}{\rotatebox[origin=c]{90}{ViT}}}& \textbf{74.9} & \underline{86.8} & 89.5 & 82.8 & \textbf{87.9} & 88.3 & 79.8 & \textbf{71.9} & \underline{90.1} & 85.5 & \underline{74.6} & 90.6 & \underline{83.6}\Tstrut \\
    \cline{1-1}\cline{3-15} 

    {CDAN}& & 62.6 & 82.9 & 87.2  & 79.2 & 84.9  & 87.1 & 77.9 & 63.3 & 88.7 & 83.1 & 63.5 & 90.8 & 79.3 \\
     \cellcolor{mygray}{CDAN w/ SDAT} && \cellcolor{mygray}69.1 & \cellcolor{mygray}86.6 & \cellcolor{mygray}88.9 & \cellcolor{mygray}81.9 & \cellcolor{mygray}86.2 & \cellcolor{mygray}88.0 & \cellcolor{mygray}\underline{81.0} & \cellcolor{mygray}66.7 & \cellcolor{mygray}89.7 & \cellcolor{mygray}86.2 & \cellcolor{mygray}72.1 & \cellcolor{mygray}\underline{91.9} & \cellcolor{mygray}82.4\Bstrut\\\cline{1-1}\cline{3-15}
    {CDAN + MCC}& & 67.0 & 84.8 & \underline{90.2} & \underline{83.4} & \underline{87.3} & \underline{89.3} & 80.7 & 64.4 & 90.0 & \underline{86.6} & 70.4 & \underline{91.9} & 82.2 \\
    \cellcolor{mygray}{CDAN + MCC w/ SDAT} && \cellcolor{mygray}\underline{70.8} & \cellcolor{mygray}\textbf{87.0} & \cellcolor{mygray}\textbf{90.5} & \cellcolor{mygray}\textbf{85.2} & \cellcolor{mygray}\underline{87.3} & \cellcolor{mygray}\textbf{89.7} & \cellcolor{mygray}\textbf{84.1} & \cellcolor{mygray}\underline{70.7} & \cellcolor{mygray}\textbf{90.6} & \cellcolor{mygray}\textbf{88.3} & \cellcolor{mygray}\textbf{75.5} & \cellcolor{mygray}\textbf{92.1} &\cellcolor{mygray} \textbf{84.3}
    \Bstrut\\\hline \hline
  \end{tabular}%
      }

  \label{tab:officehome}
\end{table*}
The bound is similar to generalization bounds for domain adaptation \citep{ben2010theory, acuna2021f}. The main difference is the sharpness aware risk term $\max_{||\epsilon|| \leq \rho}\hat{R}^l_S(h_{\theta})$ in place of source risk $R^{l}_S(h_{\theta})$, and an additional term that depends on the norm of the weights $\gamma(||\theta||_2^2/\rho^2)$. 
The first is minimized by decreasing the empirical sharpness aware source risk by using SAM loss shown in Sec.\ \ref{smoothness}. The second term is reduced by decreasing the discrepancy between source and target domains. The third term, as it is a function of norm of weights $||\theta||_2^2$, can be reduced by using either L2 regularization or weight decay. Since we assume that the $\mathcal{H}$ hypothesis class we have is rich, the $\lambda^*$ term is small. \\ \\
Any DAT baseline can be modified to use SDAT objective just by using few lines of code (App. \ref{app:pytorchcode}). \emph{We observe that the proposed SDAT objective (Eq. \ref{eq:uda_obj_smooth}) leads to significantly lower generalization error compared to the original DA objective (Eq. \ref{eq:uda_obj}), which we empirically demonstrate in the following sections}.
\section{Adaptation for classification}

We evaluate our proposed method on three datasets: Office-Home, VisDA-2017, and DomainNet, as well as by combining SDAT with two DAT based DA techniques: CDAN and CDAN+MCC. We also show results with ViT backbone on Office-Home and VisDA-2017 dataset.
\subsection{Datasets}
\textbf{Office-Home} \citep{venkateswara2017Deep}: Office-Home consists of 15,500 images from 65 classes and 4 domains: Art (Ar), Clipart (Cl), Product (Pr) and Real World (Rw).

\textbf{DomainNet }\citep{peng2019moment}: DomainNet consists of 0.6 million images across 345 classes belonging to six domains. The domains are infograph (inf), clipart (clp), painting (pnt), sketch (skt), real and quickdraw.

\textbf{VisDA-2017 }\citep{visda2017}:  VisDA is a dataset that focuses on the transition from simulation to real world and contains approximately 280K images across 12 classes. 

\subsection{Domain Adaptation Methods}
\textbf{CDAN }\citep{long2018conditional}: Conditional Domain Adversarial network is a popular DA algorithm that improves the performance of the DANN algorithm. CDAN introduces the idea of multi-linear conditioning to align the source and target distributions better. CDAN in Table \ref{tab:officehome} and \ref{table:visda_uda} refers to our implementation of CDAN* \cite{long2018conditional} method. 

\textbf{CDAN + MCC }\citep{jin2020minimum}: The minimum class confusion (MCC) loss term is added as a regularizer to CDAN. MCC is a non-adversarial term that minimizes the pairwise class confusion on the target domain, hence we consider this as an additional minimization term which is added to empirical source risk. This method achieves close to SOTA accuracy among adversarial adaptation methods. 
\subsection{Implementation Details} \label{imple}
We implement our proposed method in the Transfer-Learning-Library \citep{dalib} toolkit developed in PyTorch \citep{NEURIPS2019_9015}. The difference between the performance reported in CDAN* and our implementation CDAN is due to the batch normalization layer in domain classifier, which enhances performance.

\begin{table*}[ht!]
	\centering
	\caption{Accuracy (\%) on VisDA-2017 for unsupervised DA (with ResNet-101 and ViT backbone). The \textbf{mean} column contains mean across all classes. SDAT particularly improves the accuracy in classes that have comparatively low CDAN performance. 
}
 \vskip 0.15in
	\addtolength{\tabcolsep}{-3pt}
	\label{table:visda_uda}
	\begin{adjustbox}{max width=\textwidth}

	\begin{tabular}{l|c|cccccccccccc|c}
        \hline
		\textbf{Method} && \textbf{plane} & \textbf{bcycl} & \textbf{bus} & \textbf{car} & \textbf{horse} & \textbf{knife} & \textbf{mcyle} & \textbf{persn} & \textbf{plant} & \textbf{sktb} & \textbf{train} & \textbf{truck} & \textbf{mean} \Tstrut\Bstrut\\
		\hline \hline
		ResNet \citep{he2016deep} &\parbox[t]{2mm}{\multirow{10}{*}{\rotatebox[origin=c]{90}{ResNet-101}}}   &  55.1 & 53.3 & 61.9 & 59.1 &  80.6 & 17.9 & 79.7 & 31.2 & 81.0 & 26.5 & 73.5 & 8.5 & 52.4\\
		DANN \citep{ganin2016domain}& & 81.9 & 77.7 & 82.8 & 44.3 & 81.2 & 29.5 & 65.1 & 28.6 & 51.9 & 54.6 & 82.8 & 7.8 & 57.4\\
		MCD \citep{saito2018maximum}&&   87.0 &  60.9 & \textbf{83.7} & 64.0 & 88.9 & 79.6 & 84.7 & 76.9 & 88.6 & 40.3 & 83.0 & 25.8 & 71.9\\
		CDAN* \citep{long2018conditional}& &   85.2 & 66.9 & \underline{83.0} & 50.8 & 84.2 & 74.9 & 88.1 & 74.5 & 83.4 & 76.0 & 81.9 & 38.0 & 73.9\\
		MCC \citep{jin2020minimum} &&   88.1 & 80.3 & 80.5 & \textbf{71.5} & 90.1 & 93.2 & 85.0 & 71.6 & 89.4 & 73.8 & 85.0 & 36.9 & 78.8\\
		\cline{1-1}\cline{3-15}
		{CDAN} &&   94.9 & 72.0 & \underline{83.0} & 57.3 & 91.6 & 95.2 & \underline{91.6} & 79.5 & 85.8 & 88.8 & \underline{87.0} & 40.5 & 80.6\\
		\cellcolor{mygray}{CDAN w/ SDAT} & &  \cellcolor{mygray}94.8 & \cellcolor{mygray}77.1 & \cellcolor{mygray}82.8 & \cellcolor{mygray}60.9 & \cellcolor{mygray}92.3 & \cellcolor{mygray}95.2 & \cellcolor{mygray}\textbf{91.7} & \cellcolor{mygray}\textbf{79.9} & \cellcolor{mygray}\underline{89.9} & \cellcolor{mygray}\underline{91.2} & \cellcolor{mygray}\textbf{88.5} & \cellcolor{mygray}41.2 & \cellcolor{mygray}82.1\\
		\cline{1-1}\cline{3-15}
		{CDAN+MCC} & &  \underline{95.0} & \underline{84.2} & 75.0 & 66.9 & \textbf{94.4} & \underline{97.1} & 90.5 & \underline{79.8} & 89.4 & 89.5 & 86.9 & \underline{54.4} & \underline{83.6}\\
		\cellcolor{mygray}{CDAN+MCC w/ SDAT} &&   \cellcolor{mygray}\textbf{95.8} & \cellcolor{mygray}\textbf{85.5} & \cellcolor{mygray}76.9 &\cellcolor{mygray}\underline{69.0} & \cellcolor{mygray}\underline{93.5} & \cellcolor{mygray}\textbf{97.4} & \cellcolor{mygray}88.5 & \cellcolor{mygray}78.2 & \cellcolor{mygray}\textbf{93.1} & \cellcolor{mygray}\textbf{91.6} & \cellcolor{mygray}86.3 & \cellcolor{mygray}\textbf{55.3} & \cellcolor{mygray}\textbf{84.3}\\ \hline\hline
		TVT \cite{yang2021tvt} &\parbox[t]{2mm}{\multirow{5}{*}{\rotatebox[origin=c]{90}{ViT}}} & 92.9 & 85.6 & 77.5 & 60.5 & 93.6 & 98.2 & 89.3 & 76.4 & 93.6 & 92.0 & \underline{91.7} & 55.7 & 83.9 \\
        \cline{1-1}\cline{3-15}
        \cline{1-1}\cline{3-15}
		{CDAN} & & 94.3 & 53.0 & 75.7 & 60.5 & 93.9 & 98.3 & \textbf{96.4} & 77.5 & 91.6 & 81.8 & 87.4 & 45.2 & 79.6\\
		\cellcolor{mygray}{CDAN w/ SDAT} & &  \cellcolor{mygray}96.3 & \cellcolor{mygray}80.7 & \cellcolor{mygray}74.5 & \cellcolor{mygray}65.4 & \cellcolor{mygray}95.8 & \cellcolor{mygray}\textbf{99.5} & \cellcolor{mygray}92.0 & \cellcolor{mygray}\underline{83.7} & \cellcolor{mygray}93.6 & \cellcolor{mygray}88.9 & \cellcolor{mygray}85.8 & \cellcolor{mygray}\underline{57.2} & \cellcolor{mygray}84.5\\
		\cline{1-1}\cline{3-15}
		{CDAN+MCC} &&   \underline{96.9} & \underline{89.8} & \underline{82.2} & \underline{74.0} & \underline{96.5} & \underline{98.5} & 95.0 & 81.5 & \underline{95.4}& \underline{92.5} & 91.4 & \textbf{58.5} & \underline{87.7}\\
		\cellcolor{mygray}{CDAN+MCC w/ SDAT} && \cellcolor{mygray}\textbf{98.4} & \cellcolor{mygray}\textbf{90.9} & \cellcolor{mygray}\textbf{85.4} & \cellcolor{mygray}\textbf{82.1} & \cellcolor{mygray}\textbf{98.5} & \cellcolor{mygray}97.6 & \cellcolor{mygray}\underline{96.3} & \cellcolor{mygray}\textbf{86.1} & \cellcolor{mygray}\textbf{96.2} & \cellcolor{mygray}\textbf{96.7} & \cellcolor{mygray}\textbf{92.9} & \cellcolor{mygray}56.8 & \cellcolor{mygray}\textbf{89.8} \\ 
		\hline
		
	\end{tabular}%
	\end{adjustbox}
\end{table*}

We use a ResNet-50 backbone for Office-Home experiments and a ResNet-101 backbone for VisDA-2017 and DomainNet experiments. Additionally, we also report the performance of ViT-B/16 \cite{dosovitskiy2020image} backbone on Office-Home and VisDA-2017 datasets. All the backbones are initialised with ImageNet pretrained weights. We use a learning rate of 0.01 with batch size 32 in all of our experiments with ResNet backbone. We tune $\rho$ value in SDAT for a particular dataset split and use the same value across domains. The $\rho$ value is set to 0.02 for the Office-Home experiments, 0.005 for the VisDA-2017 experiments and 0.05 for the DomainNet experiments. More details are present in supplementary (refer App. \ref{app:experimental_deets}).

\subsection{Results}
%

\textbf{Office-Home}: For the Office-Home dataset, we compare our method with other DA algorithms including DANN, SRDC, MDD and f-DAL in Table \ref{tab:officehome}.
We can see that the addition of SDAT improves the performance on both CDAN and CDAN+MCC across majority of source and target domain pairs. CDAN+MCC w/ SDAT achieves SOTA adversarial adaptation performance on the Office-Home dataset with ResNet-50 backbone. With ViT backbone, the increase in accuracy due to SDAT is more significant. This may be attributed to the observation that ViTs reach a sharp minima compared to ResNets \cite{chen2021vision}. CDAN + MCC w/ SDAT outperforms TVT \cite{yang2021tvt}, a recent ViT based DA method and achieves SOTA results on both Office-Home and VisDA datasets (Table \ref{table:visda_uda}). Compared to the proposed method, TVT is computationally more expensive to train, contains additional adaptation modules and utilizes a backbone that is pretrained on ImageNet-21k dataset (App. \ref{app:comp_tvt}). On the other hand, SDAT is conceptually simple and can be trained on a single 12 GB GPU with ViT (pretrained on ImageNet). 
With ViT backbone, SDAT particularly improves the performance of source-target pairs which have low accuracy on the target domain (Pr$\veryshortarrow$Cl, Rw$\veryshortarrow$Cl, Pr$\veryshortarrow$Ar, Ar$\veryshortarrow$Pr in Table \ref{tab:officehome}). 

\begin{table}[H]
\caption{Results on DomainNet with CDAN w/ SDAT. The number in the parenthesis refers to the increase in Acc. w.r.t. CDAN.}
\vskip 0.15in
\begin{adjustbox}{max width=\linewidth}
\begin{tabular}{c | c  c  c  c  c | c } 
 \hline
  \textbf{Target \textbf{($\rightarrow$)}} & \multirow{2}{*}{\textbf{clp}} & \multirow{2}{*}{\textbf{inf}} & \multirow{2}{*}{\textbf{pnt}} & \multirow{2}{*}{\textbf{real}} & \multirow{2}{*}{\textbf{skt}} & \multirow{2}{*}{\textbf{Avg}} \Tstrut\\  \textbf{Source ($\downarrow$)} &&&&&& \Bstrut\\
 \hline\hline
\multirow{2}{*}{\textbf{clp}} & - & 22.0  & 41.5 & 57.5 &  47.2 & 42.1 \Tstrut\\&& \textcolor{ForestGreen}{(+1.4)}& \textcolor{ForestGreen}{(+2.6)}&
 \textcolor{ForestGreen}{(+1.5)} & \textcolor{ForestGreen}{(+2.3)} & \textcolor{ForestGreen}{(+2.0)} \\ 

\multirow{2}{*}{\textbf{inf}} & 33.9 & - & 30.3 & 48.1 & 27.9 & 35.0 \Tstrut\\& \textcolor{ForestGreen}{(+2.3)}&&
  \textcolor{ForestGreen}{(+1.0)} & \textcolor{ForestGreen}{(+4.5)} & \textcolor{ForestGreen}{(1.5)} & \textcolor{ForestGreen}{(+2.3)} \\

\multirow{2}{*}{\textbf{pnt}} &  47.5 & 20.7 & - & 58.0 & 41.8 & 42.0 \Tstrut\\ & \textcolor{ForestGreen}{(+3.4)}& 
\textcolor{ForestGreen}{(+0.9)} 
 &&
 \textcolor{ForestGreen}{(+0.8)} &
\textcolor{ForestGreen}{(+1.8)} &
 \textcolor{ForestGreen}{(+1.7)} \\

\multirow{2}{*}{\textbf{real}} &  56.7 & 25.1 & 53.6 & - & 43.9 & 44.8 \Tstrut \\ & \textcolor{ForestGreen}{(+0.9)} &
 \textcolor{ForestGreen}{(+0.7)} &
 \textcolor{ForestGreen}{(+0.4)} &
 &
 \textcolor{ForestGreen}{(+1.6)} &
 \textcolor{ForestGreen}{(+1.0)} \\

\multirow{2}{*}{\textbf{skt}} &  58.7 & 21.8 & 48.1 & 57.1 & - & 46.4 \Tstrut \\ &  \textcolor{ForestGreen}{(+2.7)} & 
 \textcolor{ForestGreen}{(+1.1)} &
 \textcolor{ForestGreen}{(+2.8)} &
 \textcolor{ForestGreen}{(+2.2)} &
&
 \textcolor{ForestGreen}{(+2.2)} 
\\\hline

\multirow{2}{*}{\textbf{Avg}} &  49.2 & 22.4 & 43.4 & 55.2 & 40.2 & 42.1 \Tstrut \\ & \textcolor{ForestGreen}{(+2.3)} &
 \textcolor{ForestGreen}{(+1.0)} &
 \textcolor{ForestGreen}{(+1.7)} &  \textcolor{ForestGreen}{(+2.2)} &
 \textcolor{ForestGreen}{(+1.8)} &
 \textcolor{ForestGreen}{(+1.8)}

\end{tabular}
\label{table:domainnet}
 \end{adjustbox}
\end{table}

\textbf{DomainNet}: Table \ref{table:domainnet} shows the results on the large and challenging DomainNet dataset across five domains. The proposed method improves the performance of CDAN significantly across all source-target pairs. On specific source-target pairs like inf $\veryshortarrow$ real, the performance increase is 4.5\%. The overall performance of CDAN is improved by nearly 1.8\% which is significant considering the large number of classes and images present in DomainNet. The improved results are attributed to stabilized adversarial training through proposed SDAT which can be clearly seen in Fig. \ref{fig:hessian}\textcolor{mydarkblue}{D}.

\textbf{VisDA-2017}: CDAN w/ SDAT improves the overall performance of CDAN by more than 1.5\% with ResNet backbone and by 4.9\% with ViT backbone on VisDA-2017 dataset (Table \ref{table:visda_uda}). Also, on CDAN + MCC baseline SDAT leads to 2.1\% improvement over baseline, leading to SOTA accuracy of 89.8\% across classes.  Particularly, SDAT significantly improves the performance of underperforming minority classes like bicycle, car and truck.  Additional baselines and results are reported in supplementary (App. \ref{app:add_results}) {along with a discussion on statistical significance (App. \ref{app:stats_sig})}.


\section{Adaptation for object detection}
To further validate our approach's generality and extensibility, we did experiments on DA for object detection. We use the same setting as proposed in DA-Faster \citep{chen2018domaindafaster} with all domain adaptation components and use it as our baseline. We use the mean Average Precision at 0.5 IoU (mAP) as our evaluation metric. In object detection, the smoothness enhancement can be achieved in two ways (empirical comparison in Sec.~\ref{sec:od_results}) :

\textbf{a) DA-Faster w/ SDAT-Classification:} Smoothness enhancement for classification loss.

\textbf{b) DA-Faster w/ SDAT:} Smoothness enhancement for the combined classification and regression loss.
\subsection{Experimental Setup}
We evaluate our proposed approach on object detection on two different domain shifts:

\textbf{Pascal to Clipart} ($P \rightarrow C$): 
Pascal \citep{everingham2010pascal} is a real-world image dataset which consists images with 20 different object categories. Clipart \citep{inoue2018crossclipart} is a graphical image dataset with complex backgrounds and has the same 20 categories as Pascal. We use ResNet-101 \citep{he2016deep} backbone  for {Faster R-CNN} \citep{ren2015faster} following \citet{saito2019strong}. 

\textbf{Cityscapes to Foggy Cityscapes} ($C \rightarrow Fc$): Cityscapes \citep{cordts2016cityscapes} is a street scene dataset for driving, whose images are collected in clear weather. Foggy Cityscapes \citep{sakaridis2018semanticfoggycityscapes} dataset is synthesized from Cityscapes for the foggy weather. We use ResNet-50 \citep{he2016deep} as the backbone for  {Faster R-CNN} for experiments on this task. Both domains have the same 8 object categories with instance labels. 


\begin{figure*}[h]
    \centering
    \begin{subfigure}[b]{0.3\linewidth}
  \centering
  \includegraphics[width=\textwidth, height = 4.1cm]{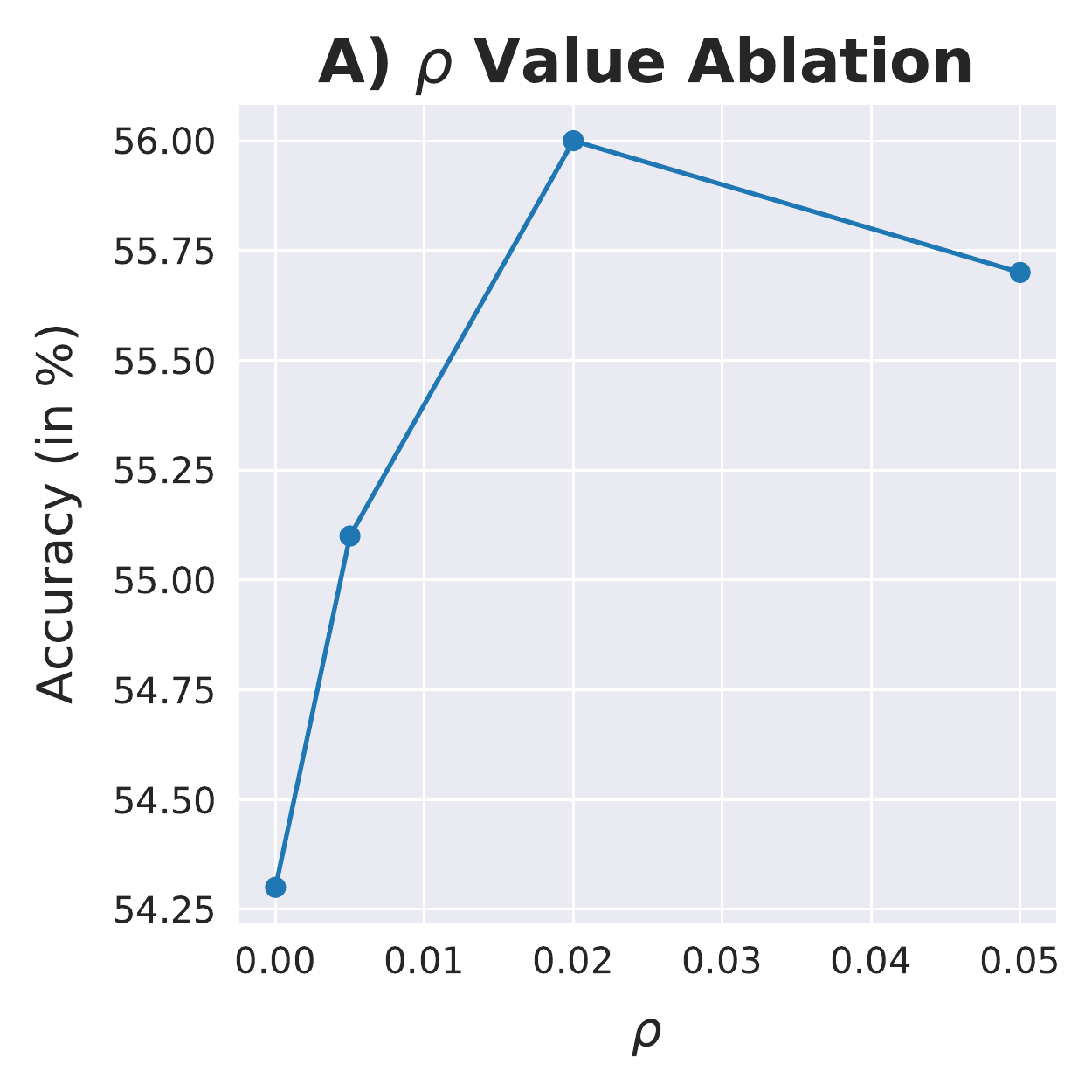}
  \label{fig:rho_ablation}
\end{subfigure}
\begin{subfigure}[b]{0.3\linewidth}
  \centering
  \includegraphics[width=\textwidth, height=4.1cm ]{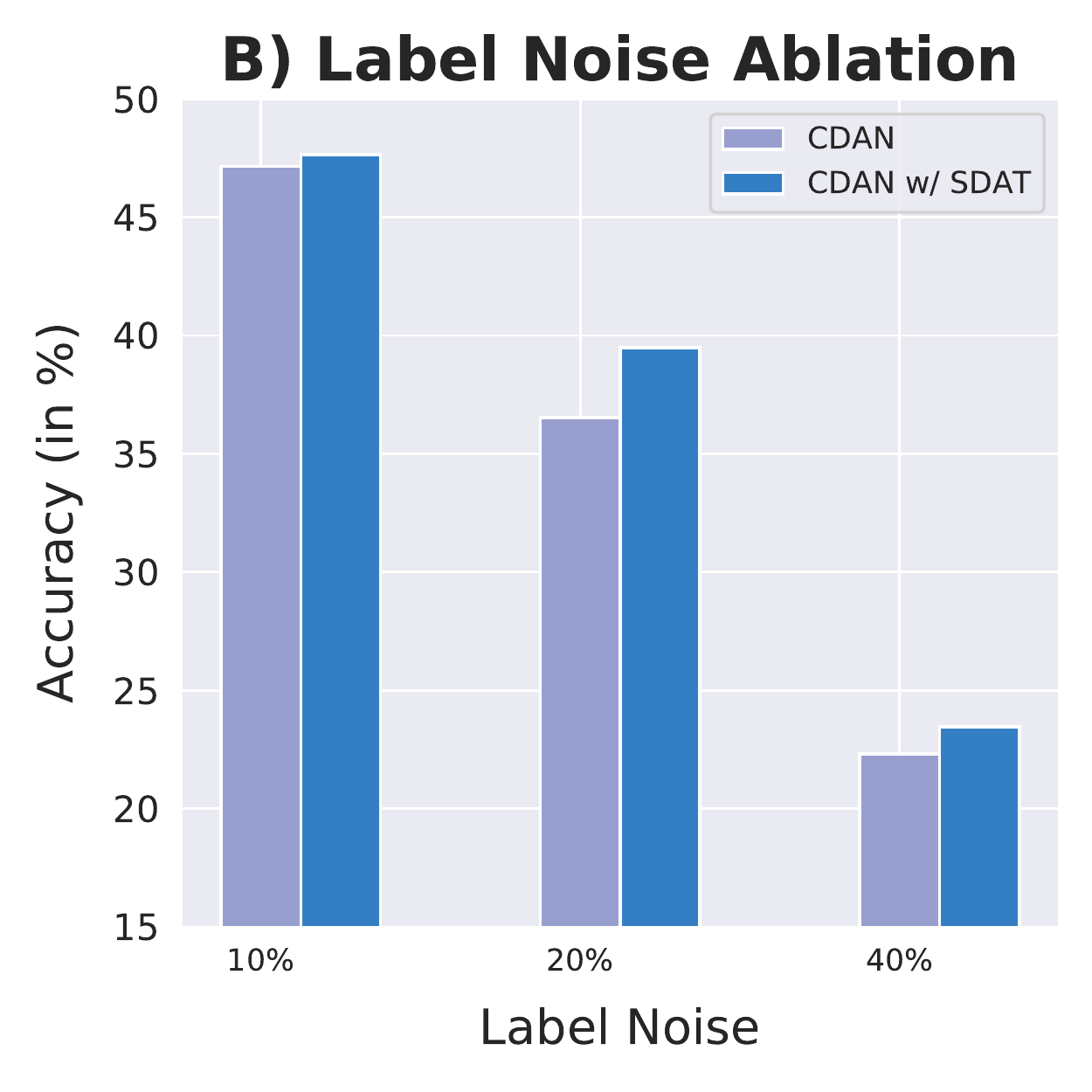}
  \label{fig:labelnoise}
\end{subfigure}
\begin{subfigure}[b]{0.3\linewidth}
  \centering
  \subfloat{
  \qquad
  \begin{adjustbox}{max width=\linewidth}
  \begin{tabular}{ccc}
 \multicolumn{3}{c}{\large \textbf{C) Smoothing Loss Components}}\\
  \hline
  \textbf{Smooth Task}&\textbf{Smooth Adv} &\textbf{Accuracy}\Tstrut\\ \hline \hline
  \textcolor{red}{\xmark} &\textcolor{red}{\xmark} & 54.3\\\hline
 \textcolor{red}{\xmark}&\textcolor{green}{\cmark}&51.0\\\hline
  \textcolor{green}{\cmark}&\textcolor{red}{\xmark}&\textbf{55.7}\\\hline  
  \textcolor{green}{\cmark}&\textcolor{green}{\cmark}&54.9\\\hline 
  \end{tabular} \end{adjustbox}}
  \vspace{5em}
\end{subfigure}\vspace{-2em}
\caption{{ Analysis of SDAT for Ar $\rightarrow$ Cl split of Office-Home dataset. \textbf{A)} Variation of target accuracy with maximum perturbation $\rho$. \textbf{B)} Comparison of accuracy of SDAT with DAT for different ratio of label noise. \textbf{C)} Comparison of accuracy when smoothing is applied to various loss components.  }}
\label{fig:smooth}
\end{figure*}

\subsection{Results}
\label{sec:od_results}
Table~\ref{tab:object_detection} shows the results on two domain shifts with varying batch size (\textit{bs}) during training. We find that only smoothing w.r.t. classification loss is much more effective (SDAT-Classification) than smoothing w.r.t. combined classification and regression loss (SDAT). On average, SDAT-Classification produces an mAP gain of {2.0}\% compared to SDAT, and {2.8}\% compared to DA-Faster baseline. 

\begin{table}[!h]
    \vspace{-0.5em}
    \caption{Results on DA for object detection. }
    \vskip 0.15in
    \centering
    \begin{adjustbox}{max width=\columnwidth}
    \begin{tabular}{l|ccc}
    \hline
    \multirow{2}{*}{Method} & $C \rightarrow Fc$  & $P \rightarrow C$ &  $P \rightarrow C$ \Tstrut\\ & \textit{ (bs=2)} & \textit{(bs=2)} &  \textit{ (bs=8)}\\
    \hline \hline
        DA-Faster \citep{chen2018domaindafaster} & 35.21 & 29.96 &  26.40 \\ 
        DA-Faster w/ SDAT & 37.47 & 29.04 &  27.64 \\ 
        DA-Faster w/ SDAT-Classification & \textbf{38.00} & \textbf{31.23} &  \textbf{30.74}
    \end{tabular}
    \end{adjustbox}
    \label{tab:object_detection}
    
 \end{table}

The proposed SDAT-Classification significantly outperforms DA-Faster baseline  and improves mAP by {1.3}\% on $P \rightarrow C$ and by {2.8}\% on $C \rightarrow Fc$. It is noteworthy that increase in performance of SDAT-Classification is consistent even after training with higher batch size $(bs=8)$ achieving improvement of {4.3}\% in mAP. Table \ref{tab:object_detection} also shows that even DA-Faster w/ SDAT (i.e. smoothing both classification and regression) outperforms DA-Faster by {0.9} \% on average. The improvement due to SDAT on adaptation for object detection shows the generality of SDAT across techniques that have some form of adversarial component present in the loss formulation. 

\section{Discussion}
\label{sec:discussion}
\textbf{How much smoothing is optimal?}: Figure \ref{fig:smooth}\textcolor{mydarkblue}{A} shows the ablation on $\rho$ value (higher $\rho$ value corresponds to more smoothing) on the Ar$\veryshortarrow$Cl from Office-Home dataset with CDAN backbone. The performance of the different values of $\rho$ is higher than the baseline with $\rho$ = 0. It can be seen that $\rho$ = 0.02 works best among all the different values and outperforms the baseline by at least 1.5\%. We found that the same $\rho$ value usually worked well across domains in a dataset, but different $\rho$ was optimal for different datasets. More details about optimum $\rho$ is in App. \ref{app:opt_rho}.

\textbf{Which components benefit from smooth optima?}:
Fig. \ref{fig:smooth}C shows the effect of introducing smoothness enhancement for different components in DAT.
For this we use SAM on a) task loss (SDAT) b) adversarial loss (SDAT w/ adv) c) both task and adversarial loss (SDAT-all). It can be seen that smoothing the adversarial loss component (SDAT w/ adv) reduces the performance to 51.0\%, which is significantly lower than even the DAT baseline.

\textbf{Is it Robust to Label Noise?}: In practical, real-world scenarios, the labeled datasets are often corrupted with some amount of label noise. Due to this, performing domain adaptation with such data is challenging. We find that smoother minima through SDAT lead to robust models which generalize well on the target domain. Figure \ref{fig:smooth}\textcolor{mydarkblue}{B} provides the comparison of SGD vs. SDAT for different percentages of label noise injected into training data.

\begin{table}
\caption{Performance comparison across different loss smoothing techniques on Office-Home. SDAT (with ResNet-50 backbone) outperforms other smoothing techniques in each case consistently.}
\vskip 0.15in
    \begin{adjustbox}{max width=\linewidth}
    \begin{tabular}{l|cccc|cc}
    \hline
    {Method} & Ar$\veryshortarrow$Cl & Cl$\veryshortarrow$Pr & Rw$\veryshortarrow$Cl &  Pr$\veryshortarrow$Cl & \multicolumn{2}{c}{{Avg}}\\
    \hline \hline
       {DAT} & 54.3 & 69.5 & 60.1 & 55.3 &  59.2 &\\
         
         {VAT} & 54.6 &70.7&  60.8 & 54.4 & 60.1 &  \textcolor{ForestGreen}{(+0.9)}\\
       {SWAD} & 54.6& 71.0 & 60.9 & 55.2 & 60.4 & \textcolor{ForestGreen}{(+1.2)} \\
        {LS}  & 53.6  & 71.6 & 59.9 &  53.4 & 59.6 & \textcolor{ForestGreen}{(+0.4)}  \\
        {SAM}  & 54.9  & 70.9 & 59.2 &  53.9 & 59.7 & \textcolor{ForestGreen}{(+0.5)}  \\
        {SDAT} &\textbf{56.0}&  \textbf{73.2} & \textbf{61.4} &\textbf{55.9} & \textbf{61.6} &\textcolor{ForestGreen}{(+2.4)}
  
    \end{tabular}
    \end{adjustbox}
    
    \label{tab:diff_smooth}
\end{table}

\textbf{Is it better than other smoothing techniques?}
To answer this question, we compare SDAT with different smoothing techniques originally proposed for ERM. We specifically compare our method against DAT, Label Smoothing (LS) \citep{szegedy2016rethinking}, SAM \cite{foret2021sharpnessaware} and VAT \citep{miyato2019vat}. \citet{stutz2021relating} recently showed that these techniques produce a  significantly smooth loss landscape in comparison to SGD. We also compare with a very recent SWAD \citep{cha2021swad} technique which is shown effective for domain generalization. For this, we run our experiments on four different splits of the Office-Home dataset and summarize our results in Table \ref{tab:diff_smooth}. We find that techniques for ERM (LS, SAM and VAT) fail to provide significant consistent gain in performance which also confirms the requirement of specific smoothing strategies for DAT. We find that SDAT even outperforms SWAD on average by a significant margin of 1.2\%. Additional details regarding the specific methods are provided in App. \ref{app:smooth_tech}.

\begin{table}[t!]
    \centering

    \setlength{\tabcolsep}{4pt}
 \caption{Analysing the effect of SDAT on DANN (on DomainNet and Office-Home with ResNet-50 and ViT-B/16 respectively) and GVB-GD (on Office-Home with ResNet-50).}
 \vskip 0.15in
    \begin{adjustbox}{max width=\columnwidth}
    \begin{tabular}{l|c|cccc}
    \hline
    {\textbf{DomainNet}} && \textbf{clp$\veryshortarrow$pnt} &  \textbf{skt$\veryshortarrow$pnt} & \textbf{inf$\veryshortarrow$real} & \textbf{skt$\veryshortarrow$clp} \\
    \hline \hline
    \textbf{DANN}&\parbox[t]{2mm}{\multirow{2}{*}{\rotatebox[origin=c]{90}{RN-50}}} & 37.5 & 43.9 & 37.7 & 53.8 \Bstrut\Tstrut{}\\
    
     \textbf{DANN w/ SDAT}&& 38.9 \textcolor{ForestGreen}{(+1.4)} & 45.7 \textcolor{ForestGreen}{(+1.8)} & 39.6 \textcolor{ForestGreen}{(+1.9)} & 56.3 \textcolor{ForestGreen}{(+2.5)} \Bstrut\Tstrut\\
     
     \hline\hline
    {\textbf{Office-Home}} && \textbf{Ar$\veryshortarrow$Cl} &  \textbf{Cl$\veryshortarrow$Pr} & \textbf{Rw$\veryshortarrow$Cl} & \textbf{Pr$\veryshortarrow$Cl} \Tstrut\\
    \hline 
        \textbf{GVB-GD}&\parbox[t]{2mm}{\multirow{4}{*}{\rotatebox[origin=c]{90}{RN-50}}} &56.4 & 74.2& 59.0& 55.9\Bstrut\Tstrut\\
    
     \textbf{GVB-GD w/ SDAT}&& 57.6 \textcolor{ForestGreen}{(+1.2)} & 75.4 \textcolor{ForestGreen}{(+1.2)}  & 60.0 \textcolor{ForestGreen}{(+1.0)}  &  56.6 \textcolor{ForestGreen}{(+0.7)}   \Bstrut\\
     
     \cline{1-1}
     \cline{3-6}

    \textbf{DANN}&\parbox[t]{2mm}{} & 52.6 & 65.4 & 60.4 & 52.3 \Bstrut\Tstrut\\
    
     \textbf{DANN w/ SDAT}&& 53.4 \textcolor{ForestGreen}{(+0.8)}  & 66.4 \textcolor{ForestGreen}{(+1.0)} & 61.3 \textcolor{ForestGreen}{(+0.9)} &  53.8 \textcolor{ForestGreen}{(+1.5)} \Bstrut\\
     
     \hline

    \textbf{DANN}&\parbox[t]{2mm}{\multirow{2}{*}{\rotatebox[origin=c]{90}{ViT}}} & 62.7&81.8&68.5&66.5 \Bstrut\Tstrut\\
    
     \textbf{DANN w/ SDAT}&& 68.0 \textcolor{ForestGreen}{(+5.3)}  & 82.4 \textcolor{ForestGreen}{(+0.6)} & 73.4 \textcolor{ForestGreen}{(+4.9)} &  68.8 \textcolor{ForestGreen}{(+2.3)} \Bstrut\\
     
     \hline\hline
     

    \end{tabular}
    
    \label{tab:dann_domainnet}
\end{adjustbox}
\end{table}
\textbf{Does it generalize well to other DA methods?}: We show results highlighting the effect of smoothness (SDAT) on DANN\cite{ganin2016domain} and GVB-GD \cite{cui2020gvb} in Table \ref{tab:dann_domainnet} with ResNet-50 and ViT backbone. DANN w/ SDAT leads to gain in accuracy on both DomainNet and Office-Home dataset. We observe a significant increase (average of +3.3\%) with DANN w/ SDAT (ViT backbone) on Office-Home dataset. SDAT leads to a decent gain in accuracy on Office-Home dataset with GVB-GD despite the fact that GVB-GD is a much stronger baseline than DANN. We primarily focused on CDAN and CDAN + MCC for the main results as we wanted to establish that SDAT can improve on even SOTA DAT methods for showing its effectiveness. Overall, we have shown results with four DA methods (CDAN, CDAN+MCC, DANN, GVB-GD) and this shows that SDAT is a generic method that can applied on top of any domain adversarial training based method to get better performance.

%

%


\section{Conclusion}
In this work, we analyze the curvature of loss surface of DAT used extensively for Unsupervised DA. We find that converging to a smooth minima w.r.t. {task} loss (i.e., empirical source risk) leads to stable DAT which results in better generalization on the target domain. We also theoretically and empirically show that smoothing adversarial components of loss lead to sub-optimal results, hence should be avoided in practice. We then introduce our practical and effective method, SDAT, which only increases the smoothness w.r.t. {task} loss, leading to better generalization on the target domain. SDAT leads to an effective increase for the latest methods for adversarial DA, achieving SOTA performance on benchmark datasets. One limitation of SDAT is presence of no automatic way of selecting $\rho$ (extent of smoothness) which is a good future direction to explore. 

\textbf{Acknowledgements}: This work was supported in part by  SERB-STAR Project (Project:STR/2020/000128), Govt. of India. Harsh Rangwani is supported by Prime Minister's Research Fellowship (PMRF). We are thankful for their support.

\bibliography{paper}

\begin{thebibliography}{62}
\providecommand{\natexlab}[1]{#1}
\providecommand{\url}[1]{\texttt{#1}}
\expandafter\ifx\csname urlstyle\endcsname\relax
  \providecommand{\doi}[1]{doi: #1}\else
  \providecommand{\doi}{doi: \begingroup \urlstyle{rm}\Url}\fi

\bibitem[Acuna et~al.(2021)Acuna, Zhang, Law, and Fidler]{acuna2021f}
Acuna, D., Zhang, G., Law, M.~T., and Fidler, S.
\newblock f-domain-adversarial learning: Theory and algorithms.
\newblock \emph{arXiv preprint arXiv:2106.11344}, 2021.

\bibitem[Adel et~al.(2019)Adel, Valera, Ghahramani, and Weller]{adel2019one}
Adel, T., Valera, I., Ghahramani, Z., and Weller, A.
\newblock One-network adversarial fairness.
\newblock In \emph{Proceedings of the AAAI Conference on Artificial
  Intelligence}, volume~33, pp.\  2412--2420, 2019.

\bibitem[Ben-David et~al.(2010)Ben-David, Blitzer, Crammer, Kulesza, Pereira,
  and Vaughan]{ben2010theory}
Ben-David, S., Blitzer, J., Crammer, K., Kulesza, A., Pereira, F., and Vaughan,
  J.~W.
\newblock A theory of learning from different domains.
\newblock \emph{Machine learning}, 79\penalty0 (1):\penalty0 151--175, 2010.

\bibitem[Berthelot et~al.(2021)Berthelot, Roelofs, Sohn, Carlini, and
  Kurakin]{berthelot2021adamatch}
Berthelot, D., Roelofs, R., Sohn, K., Carlini, N., and Kurakin, A.
\newblock Adamatch: A unified approach to semi-supervised learning and domain
  adaptation.
\newblock \emph{arXiv preprint arXiv:2106.04732}, 2021.

\bibitem[Biewald(2020)]{wandb}
Biewald, L.
\newblock Experiment tracking with weights and biases, 2020.
\newblock URL \url{https://www.wandb.com/}.
\newblock Software available from wandb.com.

\bibitem[Carmon et~al.(2020)Carmon, Duchi, Hinder, and
  Sidford]{carmon2020lower}
Carmon, Y., Duchi, J.~C., Hinder, O., and Sidford, A.
\newblock Lower bounds for finding stationary points {I}.
\newblock \emph{Mathematical Programming}, 184\penalty0 (1):\penalty0 71--120,
  2020.

\bibitem[Cha et~al.(2021)Cha, Chun, Lee, Cho, Park, Lee, and Park]{cha2021swad}
Cha, J., Chun, S., Lee, K., Cho, H.-C., Park, S., Lee, Y., and Park, S.
\newblock Swad: Domain generalization by seeking flat minima.
\newblock \emph{arXiv preprint arXiv:2102.08604}, 2021.

\bibitem[Chen et~al.(2021)Chen, Hsieh, and Gong]{chen2021vision}
Chen, X., Hsieh, C.-J., and Gong, B.
\newblock When vision transformers outperform resnets without pretraining or
  strong data augmentations.
\newblock \emph{arXiv preprint arXiv:2106.01548}, 2021.

\bibitem[Chen et~al.(2018)Chen, Li, Sakaridis, Dai, and
  Van~Gool]{chen2018domaindafaster}
Chen, Y., Li, W., Sakaridis, C., Dai, D., and Van~Gool, L.
\newblock Domain adaptive faster r-cnn for object detection in the wild.
\newblock In \emph{Proceedings of the IEEE conference on computer vision and
  pattern recognition}, pp.\  3339--3348, 2018.

\bibitem[Chu et~al.(2020)Chu, Minami, and Fukumizu]{chu2020smoothness}
Chu, C., Minami, K., and Fukumizu, K.
\newblock Smoothness and stability in gans.
\newblock \emph{arXiv preprint arXiv:2002.04185}, 2020.

\bibitem[Cordts et~al.(2016)Cordts, Omran, Ramos, Rehfeld, Enzweiler, Benenson,
  Franke, Roth, and Schiele]{cordts2016cityscapes}
Cordts, M., Omran, M., Ramos, S., Rehfeld, T., Enzweiler, M., Benenson, R.,
  Franke, U., Roth, S., and Schiele, B.
\newblock The cityscapes dataset for semantic urban scene understanding.
\newblock In \emph{Proceedings of the IEEE conference on computer vision and
  pattern recognition}, pp.\  3213--3223, 2016.

\bibitem[Cui et~al.(2020)Cui, Wang, Zhuo, Su, Huang, and Qi]{cui2020gvb}
Cui, S., Wang, S., Zhuo, J., Su, C., Huang, Q., and Qi, T.
\newblock Gradually vanishing bridge for adversarial domain adaptation.
\newblock In \emph{Proceedings of the IEEE Conference on Computer Vision and
  Pattern Recognition}, 2020.

\bibitem[Dosovitskiy et~al.(2020)Dosovitskiy, Beyer, Kolesnikov, Weissenborn,
  Zhai, Unterthiner, Dehghani, Minderer, Heigold, Gelly,
  et~al.]{dosovitskiy2020image}
Dosovitskiy, A., Beyer, L., Kolesnikov, A., Weissenborn, D., Zhai, X.,
  Unterthiner, T., Dehghani, M., Minderer, M., Heigold, G., Gelly, S., et~al.
\newblock An image is worth 16x16 words: Transformers for image recognition at
  scale.
\newblock In \emph{International Conference on Learning Representations}, 2020.

\bibitem[Dziugaite \& Roy(2017)Dziugaite and Roy]{dziugaite2017computing}
Dziugaite, G.~K. and Roy, D.~M.
\newblock Computing nonvacuous generalization bounds for deep (stochastic)
  neural networks with many more parameters than training data.
\newblock \emph{arXiv preprint arXiv:1703.11008}, 2017.

\bibitem[Everingham et~al.(2010)Everingham, Van~Gool, Williams, Winn, and
  Zisserman]{everingham2010pascal}
Everingham, M., Van~Gool, L., Williams, C.~K., Winn, J., and Zisserman, A.
\newblock The pascal visual object classes (voc) challenge.
\newblock \emph{International journal of computer vision}, 88\penalty0
  (2):\penalty0 303--338, 2010.

\bibitem[Foret et~al.(2021)Foret, Kleiner, Mobahi, and
  Neyshabur]{foret2021sharpnessaware}
Foret, P., Kleiner, A., Mobahi, H., and Neyshabur, B.
\newblock Sharpness-aware minimization for efficiently improving
  generalization.
\newblock In \emph{International Conference on Learning Representations}, 2021.
\newblock URL \url{https://openreview.net/forum?id=6Tm1mposlrM}.

\bibitem[Ganin \& Lempitsky(2015)Ganin and Lempitsky]{ganin2015unsupervised}
Ganin, Y. and Lempitsky, V.
\newblock Unsupervised domain adaptation by backpropagation.
\newblock In \emph{International conference on machine learning}, pp.\
  1180--1189. PMLR, 2015.

\bibitem[Ganin et~al.(2016)Ganin, Ustinova, Ajakan, Germain, Larochelle,
  Laviolette, Marchand, and Lempitsky]{ganin2016domain}
Ganin, Y., Ustinova, E., Ajakan, H., Germain, P., Larochelle, H., Laviolette,
  F., Marchand, M., and Lempitsky, V.
\newblock Domain-adversarial training of neural networks.
\newblock \emph{The journal of machine learning research}, 17\penalty0
  (1):\penalty0 2096--2030, 2016.

\bibitem[Ghorbani et~al.(2019)Ghorbani, Krishnan, and
  Xiao]{ghorbani2019investigation}
Ghorbani, B., Krishnan, S., and Xiao, Y.
\newblock An investigation into neural net optimization via hessian eigenvalue
  density.
\newblock In \emph{International Conference on Machine Learning}, pp.\
  2232--2241. PMLR, 2019.

\bibitem[Goodfellow et~al.(2014)Goodfellow, Pouget-Abadie, Mirza, Xu,
  Warde-Farley, Ozair, Courville, and Bengio]{goodfellow2014generative}
Goodfellow, I., Pouget-Abadie, J., Mirza, M., Xu, B., Warde-Farley, D., Ozair,
  S., Courville, A., and Bengio, Y.
\newblock Generative adversarial nets.
\newblock \emph{Advances in neural information processing systems}, 27, 2014.

\bibitem[He et~al.(2019)He, Huang, and Yuan]{he2019asymmetric}
He, H., Huang, G., and Yuan, Y.
\newblock Asymmetric valleys: beyond sharp and flat local minima.
\newblock In \emph{Proceedings of the 33rd International Conference on Neural
  Information Processing Systems}, pp.\  2553--2564, 2019.

\bibitem[He et~al.(2016)He, Zhang, Ren, and Sun]{he2016deep}
He, K., Zhang, X., Ren, S., and Sun, J.
\newblock Deep residual learning for image recognition.
\newblock In \emph{Proceedings of the IEEE conference on computer vision and
  pattern recognition}, pp.\  770--778, 2016.

\bibitem[Hochreiter \& Schmidhuber(1994)Hochreiter and
  Schmidhuber]{hochreiter1994simplifying}
Hochreiter, S. and Schmidhuber, J.
\newblock Simplifying neural nets by discovering flat minima.
\newblock \emph{Advances in neural information processing systems}, 7, 1994.

\bibitem[Hochreiter \& Schmidhuber(1997)Hochreiter and
  Schmidhuber]{hochreiter1997flat}
Hochreiter, S. and Schmidhuber, J.
\newblock Flat minima.
\newblock \emph{Neural computation}, 9\penalty0 (1):\penalty0 1--42, 1997.

\bibitem[Inoue et~al.(2018)Inoue, Furuta, Yamasaki, and
  Aizawa]{inoue2018crossclipart}
Inoue, N., Furuta, R., Yamasaki, T., and Aizawa, K.
\newblock Cross-domain weakly-supervised object detection through progressive
  domain adaptation.
\newblock In \emph{Proceedings of the IEEE conference on computer vision and
  pattern recognition}, pp.\  5001--5009, 2018.

\bibitem[Izmailov et~al.(2018)Izmailov, Podoprikhin, Garipov, Vetrov, and
  Wilson]{izmailov2018averaging}
Izmailov, P., Podoprikhin, D., Garipov, T., Vetrov, D., and Wilson, A.~G.
\newblock Averaging weights leads to wider optima and better generalization.
\newblock \emph{arXiv preprint arXiv:1803.05407}, 2018.

\bibitem[Jastrzebski et~al.(2020)Jastrzebski, Szymczak, Fort, Arpit, Tabor,
  Cho*, and Geras*]{Jastrzebski2020The}
Jastrzebski, S., Szymczak, M., Fort, S., Arpit, D., Tabor, J., Cho*, K., and
  Geras*, K.
\newblock The break-even point on optimization trajectories of deep neural
  networks.
\newblock In \emph{International Conference on Learning Representations}, 2020.
\newblock URL \url{https://openreview.net/forum?id=r1g87C4KwB}.

\bibitem[Jin et~al.(2020)Jin, Wang, Long, and Wang]{jin2020minimum}
Jin, Y., Wang, X., Long, M., and Wang, J.
\newblock Minimum class confusion for versatile domain adaptation.
\newblock In \emph{European Conference on Computer Vision}, pp.\  464--480.
  Springer, 2020.

\bibitem[Junguang~Jiang \& Long(2020)Junguang~Jiang and Long]{dalib}
Junguang~Jiang, Baixu~Chen, B.~F. and Long, M.
\newblock Transfer-learning-library.
\newblock \url{https://github.com/thuml/Transfer-Learning-Library}, 2020.

\bibitem[Kang \& Park(2020)Kang and Park]{kang2020ContraGAN}
Kang, M. and Park, J.
\newblock {ContraGAN: Contrastive Learning for Conditional Image Generation}.
\newblock 2020.

\bibitem[Keskar \& Socher(2017)Keskar and Socher]{keskar2017improving}
Keskar, N.~S. and Socher, R.
\newblock Improving generalization performance by switching from adam to sgd.
\newblock \emph{arXiv preprint arXiv:1712.07628}, 2017.

\bibitem[Keskar et~al.(2017)Keskar, Nocedal, Tang, Mudigere, and
  Smelyanskiy]{keskar2017large}
Keskar, N.~S., Nocedal, J., Tang, P. T.~P., Mudigere, D., and Smelyanskiy, M.
\newblock On large-batch training for deep learning: Generalization gap and
  sharp minima.
\newblock In \emph{5th International Conference on Learning Representations,
  ICLR 2017}, 2017.

\bibitem[Kingma \& Ba(2014)Kingma and Ba]{kingma2014adam}
Kingma, D.~P. and Ba, J.
\newblock Adam: A method for stochastic optimization.
\newblock \emph{arXiv preprint arXiv:1412.6980}, 2014.

\bibitem[Krizhevsky et~al.(2009)]{krizhevsky2009learning}
Krizhevsky, A. et~al.
\newblock Learning multiple layers of features from tiny images.
\newblock 2009.

\bibitem[Kundu et~al.(2020)Kundu, Venkat, Revanur, V, and
  Babu]{Kundu_2020_CVPR}
Kundu, J.~N., Venkat, N., Revanur, A., V, R.~M., and Babu, R.~V.
\newblock Towards inheritable models for open-set domain adaptation.
\newblock In \emph{Proceedings of the IEEE/CVF Conference on Computer Vision
  and Pattern Recognition (CVPR)}, June 2020.

\bibitem[Kundu et~al.(2021)Kundu, Kulkarni, Singh, Jampani, and
  Babu]{Kundu_2021_ICCV}
Kundu, J.~N., Kulkarni, A., Singh, A., Jampani, V., and Babu, R.~V.
\newblock Generalize then adapt: Source-free domain adaptive semantic
  segmentation.
\newblock In \emph{Proceedings of the IEEE/CVF International Conference on
  Computer Vision (ICCV)}, pp.\  7046--7056, October 2021.

\bibitem[Li et~al.(2018)Li, Pan, Wang, and Kot]{li2018domain}
Li, H., Pan, S.~J., Wang, S., and Kot, A.~C.
\newblock Domain generalization with adversarial feature learning.
\newblock In \emph{Proceedings of the IEEE Conference on Computer Vision and
  Pattern Recognition}, pp.\  5400--5409, 2018.

\bibitem[Liu et~al.(2017)Liu, Breuel, and Kautz]{liu2017unsupervised}
Liu, M.-Y., Breuel, T., and Kautz, J.
\newblock Unsupervised image-to-image translation networks.
\newblock In \emph{Advances in neural information processing systems}, pp.\
  700--708, 2017.

\bibitem[Long et~al.(2018)Long, Cao, Wang, and Jordan]{long2018conditional}
Long, M., Cao, Z., Wang, J., and Jordan, M.~I.
\newblock Conditional adversarial domain adaptation.
\newblock In \emph{Advances in Neural Information Processing Systems}, pp.\
  1645--1655, 2018.

\bibitem[Miyato et~al.(2018)Miyato, Kataoka, Koyama, and
  Yoshida]{miyato2018spectral}
Miyato, T., Kataoka, T., Koyama, M., and Yoshida, Y.
\newblock Spectral normalization for generative adversarial networks.
\newblock In \emph{International Conference on Learning Representations}, 2018.

\bibitem[{Miyato} et~al.(2019){Miyato}, {Maeda}, {Koyama}, and
  {Ishii}]{miyato2019vat}
{Miyato}, T., {Maeda}, S., {Koyama}, M., and {Ishii}, S.
\newblock Virtual adversarial training: A regularization method for supervised
  and semi-supervised learning.
\newblock \emph{IEEE Transactions on Pattern Analysis and Machine
  Intelligence}, 41\penalty0 (8):\penalty0 1979--1993, 2019.
\newblock \doi{10.1109/TPAMI.2018.2858821}.

\bibitem[Nguyen et~al.(2010)Nguyen, Wainwright, and
  Jordan]{nguyen2010estimating}
Nguyen, X., Wainwright, M.~J., and Jordan, M.~I.
\newblock Estimating divergence functionals and the likelihood ratio by convex
  risk minimization.
\newblock \emph{IEEE Transactions on Information Theory}, 56\penalty0
  (11):\penalty0 5847--5861, 2010.

\bibitem[Paszke et~al.(2019)Paszke, Gross, Massa, Lerer, Bradbury, Chanan,
  Killeen, Lin, Gimelshein, Antiga, Desmaison, Kopf, Yang, DeVito, Raison,
  Tejani, Chilamkurthy, Steiner, Fang, Bai, and Chintala]{NEURIPS2019_9015}
Paszke, A., Gross, S., Massa, F., Lerer, A., Bradbury, J., Chanan, G., Killeen,
  T., Lin, Z., Gimelshein, N., Antiga, L., Desmaison, A., Kopf, A., Yang, E.,
  DeVito, Z., Raison, M., Tejani, A., Chilamkurthy, S., Steiner, B., Fang, L.,
  Bai, J., and Chintala, S.
\newblock Pytorch: An imperative style, high-performance deep learning library.
\newblock In Wallach, H., Larochelle, H., Beygelzimer, A., d\textquotesingle
  Alch\'{e}-Buc, F., Fox, E., and Garnett, R. (eds.), \emph{Advances in Neural
  Information Processing Systems 32}, pp.\  8024--8035. Curran Associates,
  Inc., 2019.
\newblock URL
  \url{http://papers.neurips.cc/paper/9015-pytorch-an-imperative-style-high-performance-deep-learning-library.pdf}.

\bibitem[Peng et~al.(2017)Peng, Usman, Kaushik, Hoffman, Wang, and
  Saenko]{visda2017}
Peng, X., Usman, B., Kaushik, N., Hoffman, J., Wang, D., and Saenko, K.
\newblock Visda: The visual domain adaptation challenge, 2017.

\bibitem[Peng et~al.(2019)Peng, Bai, Xia, Huang, Saenko, and
  Wang]{peng2019moment}
Peng, X., Bai, Q., Xia, X., Huang, Z., Saenko, K., and Wang, B.
\newblock Moment matching for multi-source domain adaptation.
\newblock In \emph{Proceedings of the IEEE International Conference on Computer
  Vision}, pp.\  1406--1415, 2019.

\bibitem[Rangwani et~al.(2021)Rangwani, Jain, Aithal, and
  Babu]{Rangwani_2021_ICCV}
Rangwani, H., Jain, A., Aithal, S.~K., and Babu, R.~V.
\newblock S3vaada: Submodular subset selection for virtual adversarial active
  domain adaptation.
\newblock In \emph{Proceedings of the IEEE/CVF International Conference on
  Computer Vision (ICCV)}, pp.\  7516--7525, October 2021.

\bibitem[Ren et~al.(2015)Ren, He, Girshick, and Sun]{ren2015faster}
Ren, S., He, K., Girshick, R., and Sun, J.
\newblock {Faster R-CNN}: Towards real-time object detection with region
  proposal networks.
\newblock \emph{Advances in neural information processing systems},
  28:\penalty0 91--99, 2015.

\bibitem[Saito et~al.(2018{\natexlab{a}})Saito, Ushiku, Harada, and
  Saenko]{saito2018adversarial}
Saito, K., Ushiku, Y., Harada, T., and Saenko, K.
\newblock Adversarial dropout regularization.
\newblock In \emph{International Conference on Learning Representations},
  2018{\natexlab{a}}.

\bibitem[Saito et~al.(2018{\natexlab{b}})Saito, Watanabe, Ushiku, and
  Harada]{saito2018maximum}
Saito, K., Watanabe, K., Ushiku, Y., and Harada, T.
\newblock Maximum classifier discrepancy for unsupervised domain adaptation.
\newblock In \emph{Proceedings of the IEEE conference on computer vision and
  pattern recognition}, pp.\  3723--3732, 2018{\natexlab{b}}.

\bibitem[Saito et~al.(2019)Saito, Ushiku, Harada, and Saenko]{saito2019strong}
Saito, K., Ushiku, Y., Harada, T., and Saenko, K.
\newblock Strong-weak distribution alignment for adaptive object detection.
\newblock In \emph{Proceedings of the IEEE/CVF Conference on Computer Vision
  and Pattern Recognition}, pp.\  6956--6965, 2019.

\bibitem[Sakaridis et~al.(2018)Sakaridis, Dai, and
  Van~Gool]{sakaridis2018semanticfoggycityscapes}
Sakaridis, C., Dai, D., and Van~Gool, L.
\newblock Semantic foggy scene understanding with synthetic data.
\newblock \emph{International Journal of Computer Vision}, 126\penalty0
  (9):\penalty0 973--992, 2018.

\bibitem[Stutz et~al.(2021)Stutz, Hein, and Schiele]{stutz2021relating}
Stutz, D., Hein, M., and Schiele, B.
\newblock Relating adversarially robust generalization to flat minima.
\newblock \emph{arXiv preprint arXiv:2104.04448}, 2021.

\bibitem[Szegedy et~al.(2016)Szegedy, Vanhoucke, Ioffe, Shlens, and
  Wojna]{szegedy2016rethinking}
Szegedy, C., Vanhoucke, V., Ioffe, S., Shlens, J., and Wojna, Z.
\newblock Rethinking the inception architecture for computer vision.
\newblock In \emph{Proceedings of the IEEE conference on computer vision and
  pattern recognition}, pp.\  2818--2826, 2016.

\bibitem[Tang et~al.(2020)Tang, Chen, and Jia]{tang2020unsupervised}
Tang, H., Chen, K., and Jia, K.
\newblock Unsupervised domain adaptation via structurally regularized deep
  clustering.
\newblock In \emph{Proceedings of the IEEE/CVF conference on computer vision
  and pattern recognition}, pp.\  8725--8735, 2020.

\bibitem[Vapnik(2013)]{vapnik2013nature}
Vapnik, V.
\newblock \emph{The nature of statistical learning theory}.
\newblock Springer science \& business media, 2013.

\bibitem[Venkateswara et~al.(2017)Venkateswara, Eusebio, Chakraborty, and
  Panchanathan]{venkateswara2017Deep}
Venkateswara, H., Eusebio, J., Chakraborty, S., and Panchanathan, S.
\newblock Deep hashing network for unsupervised domain adaptation.
\newblock In \emph{({IEEE}) Conference on Computer Vision and Pattern
  Recognition ({CVPR})}, 2017.

\bibitem[Wang et~al.(2019)Wang, Jin, Long, Wang, and
  Jordan]{wang2019transferable}
Wang, X., Jin, Y., Long, M., Wang, J., and Jordan, M.~I.
\newblock Transferable normalization: towards improving transferability of deep
  neural networks.
\newblock In \emph{Proceedings of the 33rd International Conference on Neural
  Information Processing Systems}, pp.\  1953--1963, 2019.

\bibitem[Wightman(2019)]{rw2019timm}
Wightman, R.
\newblock Pytorch image models.
\newblock \url{https://github.com/rwightman/pytorch-image-models}, 2019.

\bibitem[Wu et~al.(2019)Wu, Kirillov, Massa, Lo, and
  Girshick]{wu2019detectron2}
Wu, Y., Kirillov, A., Massa, F., Lo, W.-Y., and Girshick, R.
\newblock Detectron2.
\newblock \url{https://github.com/facebookresearch/detectron2}, 2019.

\bibitem[Yang et~al.(2021)Yang, Liu, Xu, and Huang]{yang2021tvt}
Yang, J., Liu, J., Xu, N., and Huang, J.
\newblock Tvt: Transferable vision transformer for unsupervised domain
  adaptation.
\newblock \emph{arXiv preprint arXiv:2108.05988}, 2021.

\bibitem[Yao et~al.(2020)Yao, Gholami, Keutzer, and Mahoney]{yao2020pyhessian}
Yao, Z., Gholami, A., Keutzer, K., and Mahoney, M.~W.
\newblock Pyhessian: Neural networks through the lens of the hessian.
\newblock In \emph{2020 IEEE International Conference on Big Data (Big Data)},
  pp.\  581--590. IEEE, 2020.

\bibitem[Zhang et~al.(2019)Zhang, Liu, Long, and Jordan]{zhang2019bridging}
Zhang, Y., Liu, T., Long, M., and Jordan, M.
\newblock Bridging theory and algorithm for domain adaptation.
\newblock In \emph{International Conference on Machine Learning}, pp.\
  7404--7413. PMLR, 2019.

\end{thebibliography}
\bibliographystyle{icml2022}

\newpage
\appendix

\onecolumn

\section{Notation Table}
\label{app:notn_tab}
Table \ref{tab:notation} contains all the notations used in the paper and the proofs of theorems. 
\begin{table}[h!]
 \caption{The notations used in the paper and the corresponding meaning.}
    \centering
    \begin{tabular}{l|l}
    \hline
         \textbf{Notation} & \textbf{Meaning} \\
        \hline \hline
         $S$ & Labeled Source Data \\
         $T$ & Unlabelled Target Data \\
         $P_S$ (or $P_T$) & Source (or Target) Distribution \\
         $\mathcal{X}$ & Input space \\
         $\mathcal{Y}$ & Label space \\
         $y(\cdot)$ & Maps image to labels \\
         $h_\theta$ & Hypothesis function \\
         $R_S^l(h_\theta)$ (or $R_T^l(h_\theta)$) & Source (or Target) risk \\
         $\hat{R}_S^l(h_\theta)$ (or $\hat{R}_T^l(h_\theta)$) & Empirical Source (or Target) risk \\
         $\mathcal{H}$ & Hypothesis space \\
         $ D_{h_{\theta},\mathcal{H}}^\phi(P_S || P_T)$ & Discrepancy between two domains $P_S$ and $P_T$ \\
         $g_{\psi}$ & Feature extractor \\
         $f_{\Theta}$ & Classifier \\
         $\mathcal{D}_{\Phi}$ & Domain Discriminator \\
         $d_{S,T}^{\Phi}$ & Tractable Discrepancy Estimate\\
         $\nabla^2_{\theta} \hat{R}_S^l(h_{\theta})$ (or $H$) & Hessian of classification loss\\
         $Tr(H)$ & Trace of Hessian\\
         $\lambda_{max}$ & Maximum eigenvalue of Hessian\\
         $\epsilon$ & Perturbation \\
         $\rho$ & Maximum norm of $\epsilon$\\
         
    \end{tabular}
   
    \label{tab:notation}
\end{table}

\section{Connection of Discrepancy to $d_{S,T}^\Phi$ (Eq. \textcolor{red}{4}) in Main Paper}\label{app:discrepancy}
We refer reader to Appendix \textcolor{red}{C.2} of \citet{acuna2021f} for relation of $d_{S,T}^\Phi$. The $d_{S,T}^\Phi$ term defined in Eq. \textcolor{red}{4} given as:
\begin{equation}
    d_{S,T}^{\Phi} = \mathbb{E}_{x \sim P_S}[\log(\mathcal{D}_{\Phi}(g_{\psi}(x)))] + \mathbb{E}_{x \sim P_T}[ \log(1-\mathcal{D}_{\Phi}(g_{\psi}(x)))]
\end{equation}
The above term is exactly the Eq. \textcolor{red}{C.1} in \citet{acuna2021f} where they show that optimal $d_{S,T}^\Phi$ i.e.:
\begin{equation}
    \max_{\Phi} d_{S,T}^{\Phi} = D_{JS}(P_S||P_T) - 2\log(2)
\end{equation}
 Hence we can say from result in Eq. \textcolor{red}{4} is a consequence of Lemma 1 and Proposition 1 in \citep{acuna2021f}, assuming that $D_{\Phi}$ satisfies the constraints in Proposition 1.

\section{Proof of Theorems}\label{app:proof}
In this section we provide proofs for the theoretical results present in the paper:
\setcounter{theorem}{0}
\begin{theorem}[\textbf{Generalization bound}]
\label{th:gen-bound}
Suppose $l: \mathcal{Y} \times \mathcal{Y} \rightarrow [0,1] \subset dom \; \phi^*$. Let $h^*$ be the ideal joint classifier with error $\lambda^* = R_S^l(h^*) +  R_T^l(h^*)$. We have the following relation between source and target risk:
\begin{equation}
    R_{T}^l(h_{\theta}) \leq R_{S}^{l}(h_{\theta}) + D_{h_{\theta}, \mathcal{H}}^{\phi} (P_S || P_T) + \lambda^*
\end{equation}
\end{theorem}
\begin{proof}
We refer the reader to Theorem 2 in Appendix \textcolor{red}{B} of \citet{acuna2021f} for the detailed proof the theorem.
\end{proof}
We now introduce a Lemma for smooth functions which we will use in the proofs subsequently:
\setcounter{theorem}{0}
\begin{lemma}
\label{lem:lsmooth}
For an L-smooth function $f(w)$ the following holds where $w^*$ is the optimal minima:
$$
f(w) - f(w^*) \geq \frac{1}{2L}  || \nabla f(w) ||^2
$$
\end{lemma}
\begin{proof}
The L-smooth function by definition satisfies the following:
$$
f(w^*) \leq f(v) \leq f(w) + \nabla f(w) (v - w) + \frac{L}{2}||v - w||^2
$$
Now we minimize the upper bound wrt $v$ to get a tight bound on $f(w^*)$.
$$
D(v) = f(w) + \nabla f(w) (v - w) + \frac{L}{2}||v - w||^2
$$
after doing $\nabla_{v} D(v) = 0$ we get:
$$
v = w - \frac{1}{L} \nabla f(w) 
$$
By substituting the value of $v$ in the upper bound we get:
$$
f(w^*) \leq f(w) - \frac{1}{2L} || \nabla f(w) ||^2
$$
Hence rearranging the above term gives the desired result:
$$
f(w) - f(w^*) \geq \frac{1}{2L}  || \nabla f(w) ||^2
$$

\end{proof}

\setcounter{theorem}{1}
\begin{theorem}
\label{th:suboptimality}
For a given classifier $h_{\theta}$ and one step of (steepest) gradient ascent i.e. $\Phi' = \Phi + \eta (\nabla d_{S,T}^{\Phi}/||\nabla d_{S,T}^{\Phi}||)$ and $\Phi'' = \Phi + \eta (\nabla d_{S,T}^{\Phi}|_{\Phi + \hat{\epsilon}(\Phi)}/||\nabla d_{S,T}^{\Phi}|_{\Phi + \hat{\epsilon}(\Phi)}||)$ for maximizing
\begin{equation}
\begin{split}
         d_{S,T}^{\Phi'} - d_{S,T}^{\Phi''} \leq  \eta(1 - \cos \alpha)\sqrt{2L(d^*_{S,T} - d^{\Phi}_{S,T}) }
\end{split}
\end{equation}
where $\alpha$ is the angle between $\nabla d_{S,T}^{\Phi}$ and $\nabla d_{S,T}^{\Phi}|_{\Phi + \hat{\epsilon}(\Phi)}$. 
\end{theorem}

\begin{proof}[Proof of Theorem \ref{th:suboptimality}]
We assume that the function is $L$-smooth (the assumption of L-smoothness is the basis of many results in non-convex optimization \citep{carmon2020lower}) in terms of input $x$. As for a fixed $h_{\theta}$ as we use a reverse gradient procedure for measuring the discrepancy, only one step analysis is shown. This is because only a single step of gradient is used for estimating discrepancy $d^{\Phi}_{S,T}$ i.e. one step of each min and max optimization is performed alternatively for optimization. After this the $h_{\theta}$ is updated to decrease the discrepancy. Any differential function can be approximated by the linear approximation in case of small $\eta$:
\begin{equation}
    d^{\Phi + \eta v}_{S,T} \approx d^{\Phi}_{S,T} + \eta \nabla {{d^{\Phi^{\mathbf{T}}}_{S,T}}}v
\end{equation}
The dot product between two vectors can be written as the following function of norms and angle $\theta$ between those:
\begin{equation}
    \nabla d^{\Phi^{\mathbf{T}}}_{S,T}v = || \nabla d^{\Phi}_{S,T} || \; ||v|| \; cos \theta 
\end{equation}
The steepest value will be achieved when $\cos \theta = 1$ which is actually $v =  \frac{\nabla d^{\Phi}_{S,T}(x)}{||\nabla d^{\Phi}_{S,T}(x)||}$. Now we compare the descent in another direction $v_2 = \frac{\nabla d^{\Phi}_{S,T}|_{w + \epsilon(w)}}{||\nabla d^{\Phi}_{S,T}|_{w + \epsilon(w)}||}$ from the gradient descent. The difference in value can be characterized by:
\begin{equation}
     d^{\Phi + \eta v}_{S,T} - d^{\Phi + \eta v_2}_{S,T}= \eta ||\nabla d^{\Phi}_{S,T}||(1 - \cos \alpha) 
\end{equation}
As $\alpha$ is an angle between $\nabla d^{\Phi}_{S,T}|_{w + \epsilon(w)} \; (v_2)$ and $\nabla d^{\Phi}_{S,T}(X) \; (v)$. The suboptimality is dependent on the gradient magnitude. We use the following result to show that when optimality gap $ d^{*}_{S,T} - d^{\Phi}_{S,T}(x)$ is large the difference between two directions is also large.

For an L-smooth function  the following holds according to Lemma \ref{lem:lsmooth}:
$$
f(w) - f(w^*) \geq \frac{1}{2L}  || \nabla f(w) ||^2
$$

As we are performing gradient ascent $f(w) = -d^{\Phi}_{s,t}$, we get the following result:
$$
 (d^{*}_{S,T} - d^{\Phi}_{S,T}) \geq   \frac{1}{2L}|| \nabla d^{\Phi}_{S,T}(x) ||^2
$$
$$
2L (d^{*}_{S,T} - d^{\Phi}_{S,T} ) \geq \frac{(d^{\Phi + \eta v_2}_{S,T} - d^{\Phi + \eta v}_{S,T})^2}{(\eta(1 - \cos \alpha))^2}
$$
$$
\eta(1 - \cos \alpha)\sqrt{2L(d^*_{S,T} - d^{\Phi}_{S,T}) } \geq {(d^{\Phi'}_{S,T} - d^{\Phi''}_{S,T})}
$$
This shows that difference in value of by taking a step in direction of gradient $v$ vs taking the step in a different direction $v_2$ is upper bounded by the $ d^{*}_{S,T} - d^{\Phi}_{S,T}(x)$, hence if we are far from minima the difference can be potentially large. As we are only doing one step of gradient ascent $ d^{*}_{S,T} - d^{\Phi}_{S,T}$ will be potentially large, hence can lead to suboptimal measure of discrepancy. 
\end{proof}

\begin{theorem}
\label{th:th3}
Suppose l is the loss function, we denote $\lambda^* := R_S^l(h^*) + R_T^l(h^*)$ and let $h^*$ be the ideal joint hypothesis:
\begin{equation}
         R_{T}^l(h_{\theta}) \leq \; \max_{||\epsilon|| \leq \rho}\hat{R}_S^l(h_{\theta + \epsilon}) + D_{h_{\theta}, H}^{\phi}(P_S||P_T)  \\ + \gamma(||\theta||_2^2/\rho^2) + \lambda^* .
\end{equation}
where $\gamma: \mathbb{R}^{+} \rightarrow \mathbb{R}^{+}$ is a strictly increasing function.
\end{theorem}
\begin{proof}[Proof of Theorem \ref{th:th3}: ]
 In this case we make use of Theorem 2 in the paper sharpness aware minimization \citep{foret2021sharpnessaware} which states the following:
The source risk $R_S(h)$ is bounded using the following PAC-Bayes generalization bound for any $\rho$ with probability $1 - \delta$:
\begin{equation}
\begin{split}
      R_S(h_{\theta}) \leq \max_{||\epsilon|| \leq \rho} \hat{R}_S(h_{\theta})  +\sqrt{\frac{k\log\left(1+\frac{\|\boldsymbol{\theta}\|_2^2}{\rho^2}\left(1+\sqrt{\frac{\log(n)}{k}}\right)^2\right) + 4\log\frac{n}{\delta} + \tilde{O}(1)}{n-1}}  
\end{split}
\end{equation}
here $n$ is the training set size used for calculation of empirical risk $\hat{R}_S(h)$, $k$ is the number of parameters and $||\theta||_2$ is the norm of the weight parameters. The second term in equation can be abbreviated as $\gamma(||\theta||_2)$. Hence,
\begin{equation}
    R_S(h_{\theta}) \leq     \max_{||\epsilon|| \leq \rho} \hat{R}_S(h_{\theta}) + \gamma(||\theta||_2^2/\rho^2) 
\end{equation}
From the generalization bound for domain adaptation for any f-divergence  \citep{acuna2021f} (Theorem 2) we have the following result.
\begin{equation}
R_T^l(h_{\theta}) \leq R_{S}^l(h_{\theta}) + \mathcal{D}_{h_{\theta}, H}^{\phi}(P_S||P_T) + \lambda^*
\end{equation}
Combining the above two inequalities gives us the required result we wanted to prove i.e.
\begin{equation}
         R_{T}^l(h_\theta) \leq \; \tilde{R}_S^l(h_{\theta}) + D_{h_{\theta}, H}^{\phi}(P_S||P_T)  + \gamma(||\theta||_2^2/\rho^2) + \lambda^* .
\end{equation}

\end{proof}
\begin{table*}[!tb]

    \begin{minipage}{.4\linewidth}
      \caption{Architecture used for feature classifier and Domain classifier. $C$ is the number of classes. Both classifiers will take input from feature generator ($g_\theta$).}
      \vskip 0.15in
      \label{tab:clf}
      \centering
    \begin{tabular}{c||c}
    \hline
      Layer  &  Output Shape\\
      \hline
      \multicolumn{2}{c}{\textbf{Feature Classifier ($f_{\Theta}$)}} \\
      \hline
        - & Bottleneck Dimension \\
        Linear & $C$ \\
        \hline
        \multicolumn{2}{c}{\textbf{Domain Classifier} ($\mathcal{D}_\Phi$)} \\
        \hline
        - & Bottleneck Dimension \\
        Linear & 1024 \\
        BatchNorm & 1024 \\
        ReLU & 1024 \\
        Linear & 1024 \\
        BatchNorm & 1024 \\
        ReLU & 1024 \\
        Linear & 1 \\
       
    \end{tabular}
    \end{minipage}%
     \hspace{2em}
    \begin{minipage}{.4\linewidth}
      \centering
        \caption{ Accuracy (\%) on {VisDA-2017} (ResNet-101 and ViT backbone).}
        \vskip 0.15in
        \label{table:visda}
        \begin{tabular}{l|c|c}
        \hline
		\textbf{Method}& &\textbf{Synthetic $\rightarrow$ Real} \Bstrut\\
		\hline \hline
		DANN \citep{ganin2016domain} &\parbox[t]{2mm}{\multirow{7}{*}{\rotatebox[origin=c]{90}{ResNet-101}}}& 57.4 \\
		MCD \citep{saito2018maximum}  && 71.4\\
		CDAN* \citep{long2018conditional} && 73.7
		\Bstrut 
		\\
		\cline{1-1}\cline{3-3}
		CDAN && 76.6
		\\
		CDAN w/ SDAT && 78.3\Bstrut
		\\
		\cline{1-1}\cline{3-3}
		CDAN+MCC \citep{jin2020minimum} && \underline{80.4}
		\\
		CDAN+MCC w/ SDAT && \textbf{81.2}
		\\
		\hline \hline
		CDAN &\parbox[t]{2mm}{\multirow{4}{*}{\rotatebox[origin=c]{90}{ViT}}}& 76.7
		\\
		CDAN w/ SDAT && 81.1\Bstrut
		\\
		\cline{1-1}\cline{3-3}
		CDAN+MCC \citep{jin2020minimum} && \underline{85.1}
		\\
		CDAN+MCC w/ SDAT && \textbf{87.8}		
	\end{tabular}%
    \end{minipage} 
\end{table*}

\section{Hessian Analysis}\label{app:hess}
We use the PyHessian library \citep{yao2020pyhessian} to calculate the Hessian eigenvalues and the Hessian Eigen Spectral Density. For Office-Home experiments, all the calculations are performed using 50\% of the source data at the last checkpoint. For DomainNet experiments (Fig. \ref{fig:hessian}D), we use 10\% of the source data for Hessian calculation. The Maximum Eigenvalue is calculated at the checkpoint with the best validation accuracy ($\lambda_{\max}^{best}$) and the last checkpoint ($\lambda_{\max}^{last}$). Only the source class loss is used for calculating to clearly illustrate our point. {The partition was selected randomly, and the same partition was used across all the runs. We also made sure to use the same environment to run all the Hessian experiments. A subset of the data was used for Hessian calculation mainly because the hessian calculation is computationally expensive \citep{yao2020pyhessian}. This is commonly done in hessian experiments. For example, \citep{chen2021vision} (refer Appendix D) uses 10\% of training data for Hessian Eigenvalue calculation}. 
The PyHessian library uses Lanczos algorithm \citep{ghorbani2019investigation} for calculating the Eigen Spectral density of the Hessian and uses the Hutchinson method to calculate the trace of the Hessian efficiently. 

\section{Smoothness of Discriminator in SNGAN}
\label{app:gan_exp}
\begin{figure*}[!t]
  \centering
  \includegraphics[scale=0.5]{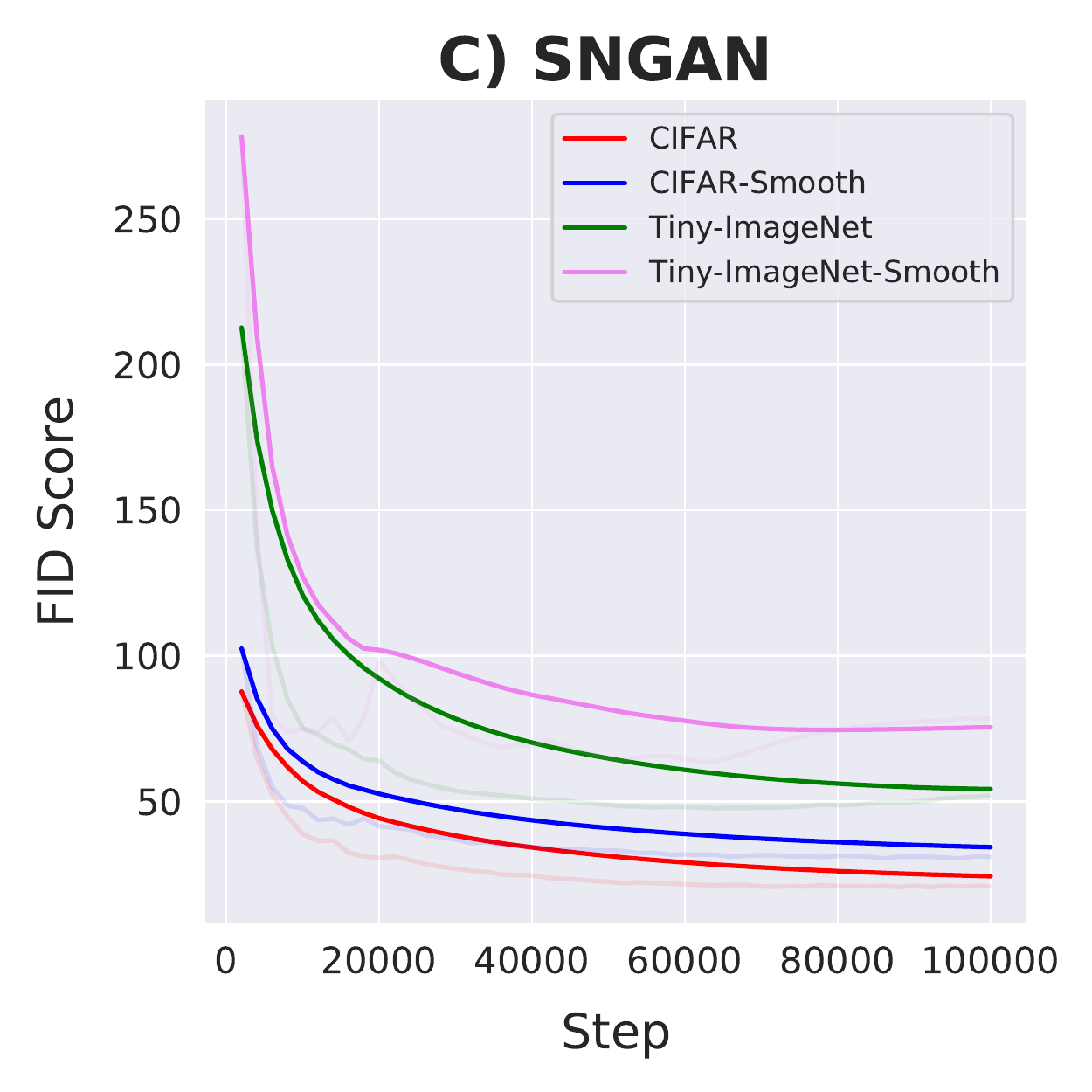}
  \caption{SNGAN performance on different datasets, smoothing discriminator in GAN also leads to inferior GAN performance (higher FID) across both datasets.}
  \label{fig:gan}
\end{figure*}
For further establishing the generality of sub-optimality of smooth adversarial loss, we also perform experiments on Spectral Normalised Generative Adversarial Networks (SNGAN) \citep{miyato2018spectral}. In case of SNGAN we also find that smoothing discriminator through SAM leads to suboptimal performance (higher FID) as in Fig. \ref{fig:gan}. The above evidences indicates that \textit{smoothing the adversarial loss leads to sub-optimality}, hence it should not be done in practice. 
We use the same configuration for SNGAN as described in PyTorch\-StudioGAN \citep{kang2020ContraGAN}  for both CIFAR10 \citep{krizhevsky2009learning} and TinyImageNet \footnote{https://www.kaggle.com/c/tiny-imagenet} with batch size of 256 in both cases. We then smooth the discriminator while discriminator is trained by using the same formulation as in Eq. \ref{eq:smooth_disc}. We find that smoothing discriminator leads to higher (suboptimal) Fr\'{e}chet Inception Distance in case of GANs as well, shown in Fig. \ref{fig:gan}.

\section{Experimental Details}\label{app:experimental_deets}
\label{exp_details}
\subsection{Image Classification}
\textbf{Office-Home}: For CDAN methods with ResNet-50 backbone, we train the models using mini-batch stochastic gradient descent (SGD) with a batch size of 32 and a learning rate of 0.01. The learning rate schedule is the same as \citep{ganin2016domain}. We train it for a total of 30 epochs with 1000 iterations per epoch. The momentum parameter in SGD is set to 0.9 and a weight decay of 0.001 is used. For CDAN+MCC experiments with ResNet-50 backbone, we use a temperature parameter \citep{jin2020minimum} of 2.5. The bottleneck dimension for the features is set to 2048.

\textbf{VisDA-2017}:
We use a ResNet-101 backbone initialized with ImageNet weights for VisDA-2017 experiments. Center Crop is also used as an augmentation during training. We use a bottleneck dimension of 256 for both algorithms.
For CDAN runs, we train the model for 30 epochs with same optimizer setting as that of Office-Home.
For CDAN+MCC runs, we use a temperature parameter of 3.0 and a learning rate of 0.002. 

\textbf{DomainNet}:
We use a ResNet-101 backbone initialized with ImageNet weights for DomainNet experiments.  We run all the experiments for 30 epochs with 2500 iterations per epoch. The other parameters are the same as that of Office-Home.

Additional experiments with a ViT backbone are performed on Office-Home and VisDA-2017 datasets. We use the ViT-B/16 architecture pretrained on ImageNet-1k, the implementation of which is borrowed from \cite{rw2019timm}. For all CDAN runs on Office-Home and VisDA, we use an initial learning rate of 0.01, whereas for CDAN+MCC runs, the initial learning rate of 0.002 is used. $\rho$ value of 0.02 is shared across all the splits on both the datasets for the ViT backbone. A batch-size of 24 is used for Office-Home and 32 for VisDA-2017. 

To show the effectiveness of SDAT fairly and promote reproducibility, we run with and without SDAT on the same GPU and environment and with the same seed. All the above experiments were run on Nvidia V100, RTX 2080 and RTX A5000 GPUs. We used Wandb \citep{wandb} to track our experiments. We will be releasing the code to promote reproducible research.
\subsubsection{Architecture of Domain Discriminator}
One of the major reasons for increased accuracy in Office-Home baseline CDAN compared to reported numbers in the paper is the architecture of domain classifier. The main difference is the use of batch normalization layer in domain classifier, which was done in the library \citep{dalib}. Table \ref{tab:clf} shows the architecture of the feature classifier and domain classifier.

\subsection{Additional Implementations Details for DA for Object detection}
In SDAT, we modified the loss function present in \citet{chen2018domaindafaster} by adding classification loss smoothing, i.e. smoothing classification loss of RPN and ROI, used in {Faster R-CNN} \citep{ren2015faster}, by training with source data. Similarly, we applied smoothing to regression loss and found it to be less effective. We implemented SDAT for object detection using Detectron2 \citep{wu2019detectron2}. The training is done via SGD with momentum 0.9 for 70k iterations with the learning rate of 0.001, and then dropped to 0.0001 after 50k iterations. We split the target data into train and validation sets and report the best mAP on validation data. 
We fixed $\rho$ to 0.15 for object detection experiments.

\section{Additional Results}\label{app:add_results}
\textbf{VisDA-2017}:
Table \ref{table:visda} shows the overall accuracy on the VisDA-2017 with ResNet-101 and ViT backbone. The accuracy reported in this table is the overall accuracy of the dataset, whereas the accuracy reported in the Table \textcolor{red}{5} of the main paper refers to the mean of the accuracy across classes. CDAN w/ SDAT outperforms CDAN by 1.7\% with ResNet-101 and by 4.4\% with ViT backbone, showing the effectiveness of SDAT in large scale Synthetic $\rightarrow$ Real shifts. With CDAN+MCC as the DA method, adding SDAT improves the performance of the method to 81.2\% with ResNet-101 backbone.\\
\textbf{DomainNet}:
Table \ref{tab:domainnet} shows the results of the proposed method on DomainNet across five domains. We compare our results with ADDA and MCD and show that CDAN achieves much higher performance on DomainNet compared to other techniques. It can be seen that CDAN w/ SDAT further improves the overall accuracy on DomainNet by 1.8\%. 
\\
We have shown results with three different domain adaptation algorithms namely DANN \citep{ganin2015unsupervised}, CDAN \citep{long2018conditional} and CDAN+MCC \citep{jin2020minimum}. SDAT has shown to improve the performance of all the three DA methods. This shows that SDAT is a generic method that can applied on top of any domain adversarial training based method to get better performance.

 \begin{table*}[tb]

    \centering
    \caption{Accuracy(\%) on \textbf{DomainNet} dataset for unsupervised domain adaptation (ResNet-101) across five distinct domains. The row indicates the source domain and the columns indicate the target domain.}
\vskip 0.15in
    \footnotesize
    \begin{adjustbox}{max width=\textwidth} 
    \begin{tabular}{c|c c c c c c ||c|c c c c c c  }
    \hline
    ADDA & clp & inf & pnt  & rel & skt & Avg & MCD &    clp & inf & pnt  & rel & skt & Avg\\
    \hline
    clp & - & 11.2 & 24.1 & 41.9 & 30.7 & 27.0 & clp & - & 14.2 & 26.1 & 45.0 & 33.8 & 29.8\\
    inf & 19.1 & - & 16.4 & 26.9 & 14.6 & 19.2 & inf & 23.6 & - & 21.2 & 36.7 & 18.0 & 24.9\\
    pnt & 31.2 & 9.5 & - & 39.1 & 25.4 & 26.3  & pnt & 34.4 & 14.8 & - & 50.5 & 28.4 & 32.0\\
    rel & 39.5 & 14.5 & 29.1 & - & 25.7 & 	27.2 & rel & 42.6 & 19.6 & 42.6 & - & 29.3 & 33.5\\
    skt & 35.3 & 8.9 & 25.2 & 37.6 & - & 	26.7 & skt & 41.2 & 13.7 & 27.6 & 34.8 & - & 29.3 \\
    Avg & 31.3 & 11.0 & 23.7 & 36.4 & 24.1 &25.3 & Avg & 35.4 & 	15.6 & 	29.4 & 41.7 & 27.4  & 29.9\\
    \hline
    \textbf{CDAN} &   clp & inf & pnt  & rel & skt & Avg & \textbf{CDAN w/ SDAT}&    clp & inf & pnt  & rel & skt & Avg  \\\hline
    clp &     -  & 20.6& 38.9 &  56.0 & 44.9 & 40.1& clp &     -  & 22.0 & 41.5 &   57.5 & 47.2 & 42.1 \\
    inf &    31.5 &  -  & 29.3   & 43.6 & 26.3 & 32.7  & inf &    33.9 &  -  & 30.3 &  48.1 & 27.9 & 35.0\\
    pnt &    44.1 & 19.8 &   - & 57.2 & 39.9 & 40.2 & pnt &    47.5 & 20.7 &   - &   58.0 & 41.8 & 42.0\\
    rel &    55.8& 24.4 & 53.2  &  -  & 42.3 & 43.9  & rel &    56.7 & 25.1 & 53.6 &  -  & 43.9 & 44.8\\
     skt &    56.0 & 20.7 & 45.3 &   54.9&  -  & 44.2 & skt &    58.7 & 21.8 & 48.1  & 57.1 &  -  & 46.4 \\
     Avg&    46.9 & 21.4 & 41.7    & 52.9 & 38.3 & 40.2 & Avg &   49.2 & 22.4 & 43.4 &  55.2 & 40.2 & \textbf{42.1}\\

    \hline

    \hline 
\end{tabular}
\end{adjustbox}
    \label{tab:domainnet}

\end{table*}

\textbf{Source-only}: Source-only setting measures the performance of a model trained only on source domain directly on unseen target data with no further target adaptation. We compare the performance of models with and without smoothing the loss landscape for source-only experiments on VisDA-2017 (Table \ref{table:visda_uda_vit_erm}) and Office-Home (Table \ref{tab:officehome_erm}) datasets with a ViT backbone pretrained on ImageNet. Initial learning rate of 0.001 and 0.002 is used for Office-Home and VisDA-2017 dataset, respectively. $\rho$ value of 0.002 is used for ERM w/SAM run for both the datasets. It can be seen that ERM w/ SAM does not directly lead to better performance on the target domain.

\begin{table*}[hbt!]
\caption{Accuracy (\%) of source-only model trained with SGD (ERM) and SAM (ERM w/SAM) on VisDA-2017 for unsupervised DA with ViT-B/16 backbone }
\vskip 0.15in
	\centering
	\label{table:visda_uda_vit_erm}
	\begin{adjustbox}{max width=\textwidth}

	\begin{tabular}{l|cccccccccccc|c}
        \hline
		\textbf{Method} & \textbf{plane} & \textbf{bcybl} & \textbf{bus} & \textbf{car} & \textbf{horse} & \textbf{knife} & \textbf{mcyle} & \textbf{persn} & \textbf{plant} & \textbf{sktb} & \textbf{train} & \textbf{truck} & \textbf{mean} \Tstrut\Bstrut\\
		\hline \hline

		{ERM} & 98.4  & 58.3 & 80.2 & 60.7 & 89.3 & 53.6 & 88.4 & 40.8 & 62.8 & 87.4 & 94.7 & 19.1 & 69.5\\
		ERM w/ SAM  & 98.6 & 33.1 & 80.0 & 76.9 & 90.1 & 35.9 & 94.2 & 22.8 & 77.8 & 89.0 & 95.3 & 11.6 & 67.1 \\
 
		\hline\hline
		
	\end{tabular}%
	\end{adjustbox}
	
\end{table*}

\begin{table*}[hbt!]
  \centering     
  \caption{Accuracy (\%) of source-only model trained with SGD (ERM) and SAM (ERM w/SAM) on Office-Home for unsupervised DA with ViT-B/16 backbone}  
  \vskip 0.15in
  \resizebox{\textwidth}{!}{%
  \begin{tabular}{l|cccccccccccc|c}
    \hline
    \textbf{Method} & \textbf{Ar$\veryshortarrow$Cl} & \textbf{Ar$\veryshortarrow$Pr} & \textbf{Ar$\veryshortarrow$Rw} & \textbf{Cl$\veryshortarrow$Ar} & \textbf{Cl$\veryshortarrow$Pr} & \textbf{Cl$\veryshortarrow$Rw} & \textbf{Pr$\veryshortarrow$Ar} & \textbf{Pr$\veryshortarrow$Cl} & \textbf{Pr$\veryshortarrow$Rw} & \textbf{Rw$\veryshortarrow$Ar} & \textbf{Rw$\veryshortarrow$Cl} & \textbf{Rw$\veryshortarrow$Pr} & \textbf{Avg} \Tstrut
    \Bstrut\\\hline \hline
	{ERM} & 51.5 & 80.8 & 86.0 & 74.8 & 80.2 & 82.6 & 71.8 & 51.0  & 85.5 & 79.5 & 55.0 & 87.9 & 73.9 \\
	{ERM w/ SAM} & 50.8 & 79.5 & 85.2 & 72.6 & 78.4 & 81.4 & 71.8 & 49.6 & 85.2 & 79.0 & 52.8 & 87.2 & 72.8 \\
\hline \hline
  \end{tabular}%
      }

  \label{tab:officehome_erm}
\end{table*}

\section{Different Smoothing Techniques}\label{app:smooth_tech}
\textbf{Stochastic Weight Averaging (SWA)} \citep{izmailov2018averaging}: SWA is a widely popular technique to reach a flatter minima. The idea behind SWA is that averaging weights across epochs leads to better generalization because it reaches a wider optima. The recently proposed SWA-Densely (SWAD) \citep{cha2021swad} takes this a step further and proposes to average the weights across iterations instead of epochs. SWAD shows improved performance on domain generalization tasks.
We average every 400 iterations in the SWA instead of averaging per epochs. We tried averaging across 800 iterations as well and the performance was comparable. 
\\
\textbf{Difference between SWAD and SDAT}:
As SWAD performs Weight Averaging, it is not possible to selectively smooth only minimization (ERM) components with SWAD, as gradients for both the adversarial loss and ERM update weights of the backbone. Due to this, SWAD cannot reach optimal performance for DAT. For verifying this, we also compare our method by implementing SWAD for Domain Adaptation on four different source-target pairs of Office-Home dataset in Table \ref{tab:diff_smooth}. On average, SDAT (Ours) gets 61.6\% (+2.4\% over DAT) accuracy in comparison to 60.4\% (+1.2\% over DAT) for SWAD.
\\
\textbf{Virtual Adversarial Training (VAT)} \citep{miyato2019vat}: VAT is regularization technique which makes use of adversarial perturbations. Adversarial perturbations are created using Algo. 1 present in  \citep{miyato2019vat}. We added VAT by optimizing the following objective: 
\begin{equation}
\min_{\theta}  \mathbb{E}_{x \sim P_S}[\underset{||r|| \leq \epsilon}{\max} D_{KL}(h_\theta(x)|| h_{\theta}(x+r))]
\end{equation}
This value acts as a negative measure of smoothness and minimizing this will make the model smooth. For training, we set hyperparameters $\epsilon$ to 15.0, 
$\xi$ to 1e-6, and $\alpha$ as 0.1.\\
\textbf{Label Smoothing (LS)} \citep{szegedy2016rethinking}: The idea behind label smoothing is to have a distribution over outputs instead of one hot vectors. Assuming that there are $k$ classes, the correct class gets a probability of 1 - $\alpha$ and the other classes gets a probability of $\alpha/(k-1)$ . \citep{stutz2021relating} mention that label smoothing tends to avoid sharper minima during training. We use a smoothing parameter ($\alpha$) of 0.1 in all the experiments in Table \ref{tab:diff_smooth_appendix}. We also show results with smoothing parameter of 0.2 and observe comparable performance. We observe that label smoothing slightly improves the performance over DAT. \\
\textbf{SAM} \cite{foret2021sharpnessaware}: In this method, we apply SAM directly to both the task loss and adversarial loss with $\rho$ = 0.05 as suggested in the paper. It can be seen that the performance improvement of SAM over DAT is minimal, thus indicating the need for SDAT.

\begin{table*}[h!]
    \centering
    \caption{Different Smoothing techniques. We refer to \citep{stutz2021relating} to compare the proposed SDAT with other techniques to show the efficacy of SDAT. It can be seen that SDAT outperforms the other smoothing techniques significantly. Other smoothing techniques improve upon the performance of DAT showing that smoothing is indeed necessary for better adaptation.}
    \vskip 0.15in
     \begin{adjustbox}{max width=\columnwidth}
    \begin{tabular}{l|cccc}
    \hline
    {Method} & Ar$\veryshortarrow$Cl &  Cl$\veryshortarrow$Pr & Rw$\veryshortarrow$Cl &  Pr$\veryshortarrow$Cl \\
    \hline \hline
    {DAT} & 54.3 & 69.5 & 60.1 & 55.3\\
    
     {VAT} & 54.6 & 70.7 & 60.8 & 54.4 \\
      {SWAD-400} & 54.6 & 71.0 & 60.9 & 55.2 \\
       {LS ($\alpha$ = 0.1)} & 53.6 & 71.6 & 59.9 & 53.4\\
       {LS ($\alpha$ = 0.2)} & 53.5 & 71.2 & 60.5 & 53.2\\
       {SDAT} & \textbf{55.9} & \textbf{73.2} & \textbf{61.4} & \textbf{55.9} \\
    \end{tabular}
    
    \label{tab:diff_smooth_appendix}
    \end{adjustbox}
\end{table*}
\section{
{Optimum $\rho$ value}}
\label{app:opt_rho}
Table \ref{tab:rho_domainnet} and \ref{tab:rho_visda} show that $\rho$ = 0.02 works robustly across experiments providing an increase in performance (although it does not achieve the best result each time) and can be used as a rule of thumb.
\begin{table*}[h!]
    \centering
    \caption{{$\rho$ value for DomainNet}}
    \vskip 0.15in
     \begin{adjustbox}{max width=\columnwidth}
    \begin{tabular}{l|ccc}
    \hline
    {Split} & DAT &  SDAT($\rho$ = 0.02) & SDAT - Reported ($\rho$ = 0.05) \\
    \hline \hline
    \textbf{clp$\veryshortarrow$skt } & 44.9 & 46.7 & 47.2\\
    
     \textbf{skt$\veryshortarrow$clp} & 56.0 & 59.0 & 58.7 \\
      \textbf{skt$\veryshortarrow$pnt} & 45.3 & 47.8 & 48.1 \\
       \textbf{inf$\veryshortarrow$rel} & 43.6 & 47.3 & 48.1\\
    \end{tabular}
    
    \label{tab:rho_domainnet}
    \end{adjustbox}
\end{table*}
\begin{table*}[h!]
    \centering
    \caption{{$\rho$ value for VisDA-2017 Synthetic $\veryshortarrow$ Real} }
    \vskip 0.15in
     \begin{adjustbox}{max width=\columnwidth}
    \begin{tabular}{l|ccc}
    \hline
    {Backbone} & DAT &  SDAT ($\rho$ = 0.02) & SDAT Reported($\rho$ = 0.005)\\
    \hline \hline
    {CDAN} & 76.6 & 78.2 & 78.3\\
     {CDAN+MCC} & 80.4 & 80.9 & 81.2\\
    \end{tabular}
    
    \label{tab:rho_visda}
    \end{adjustbox}
\end{table*}

\section{
{Comparison with TVT}}
\label{app:comp_tvt}

TVT \cite{yang2021tvt} is a recent work that reports performance higher than the other contemporary unsupervised DA methods on the publicly available datasets. This method uses a ViT backbone and focuses on exploiting the intrinsic properties of ViT to achieve better results on domain adaptation. 
Like us, TVT uses an adversarial method for adaptation to perform well on the unseen target data. On the contrary, they introduce additional modules within their architecture. The Transferability Adaption Module (TAM) is introduced to assist the ViT backbone in capturing both discriminative and transferable features. Additionally, the Discriminative Clustering Module (DCM) is used to perform discriminative clustering to achieve diverse and clustered features. 

Even without using external modules to promote the transferability and discriminability in the features learned using ViT, we are able to report higher numbers than TVT. This advocates our efforts to show the efficacy of converging to a smooth minima w.r.t. task loss to achieve better domain alignment. Moreover, TVT uses a batch size of 64 to train the network, causing a memory requirement of more than 35GB for efficient training, which is significantly higher than the 11.5GB memory used by our method on a batch-size of 24 for Office-Home to obtain better results. This allows our method to be trained using a standard 12GB GPU, removing the need of an expensive hardware. The ViT backbone used by TVT is pretrained on a much larger ImageNet-21k dataset, whereas we use the backbone pretrained on ImageNet-1k dataset.   

\section{
{Significance and Stability of Empirical Results}}
\label{app:stats_sig}
To establish the empirical results' soundness and reliability, we run a subset of experiments (representative of each different source domain) on DomainNet. The experiments are repeated with three different random seeds leading to overall 36 experimental runs (18 for CDAN w/ SDAT (Our proposed method) and 18 for CDAN baseline). 
Due to the large computational complexity of each experiment ($\approx$20 hrs each), we have presented results for multiple trials on a subset of splits. We find (in Table \ref{tab:seed_exp}) that our method can outperform the baseline average in each of the 6 cases, establishing significant improvement across all splits. However, we found that due to the large size of DomainNet, the average increase (across three different trials) is close to the reported increase in all cases (Table \ref{tab:seed_exp}), which also serves as evidence of the soundness of reported results (for remaining splits). We also present additional statistics below for establishing soundness.

If the proposed method is unstable, there is a large variance in the validation accuracy across epochs. For analyzing the stability of SDAT, we show the validation accuracy plots in Figure \ref{fig:plotss}  on six different splits of DomainNet. We find that our proposed SDAT improves over baselines consistently across epochs without overlap in confidence intervals in later epochs. This also provides evidence for the authenticity and stability of our results.  We also find that in some cases, like when using the Infographic domain as a source, our proposed SDAT also significantly \textit{stabilizes the training} (Figure \ref{fig:plotss} \textit{inf} $\veryshortarrow$ \textit{clp}).

One of the other ways of reporting results reliably proposed by the concurrent work \citep{berthelot2021adamatch} (Section 4.4) involves reporting the median of accuracy across the last few checkpoints. The median is a measure of central tendency which ignores outlier results. We also report the median of validation accuracy for our method \emph{across all splits} for the last five epochs. It is observed that we observe similar gains for median accuracy (in Table \ref{table:domainnet_median}) as reported in Table \ref{table:domainnet}.
\begin{figure*}[t!]
    \centering
    \begin{subfigure}[b]{0.3\linewidth}
  \centering
  \includegraphics[width=\textwidth]{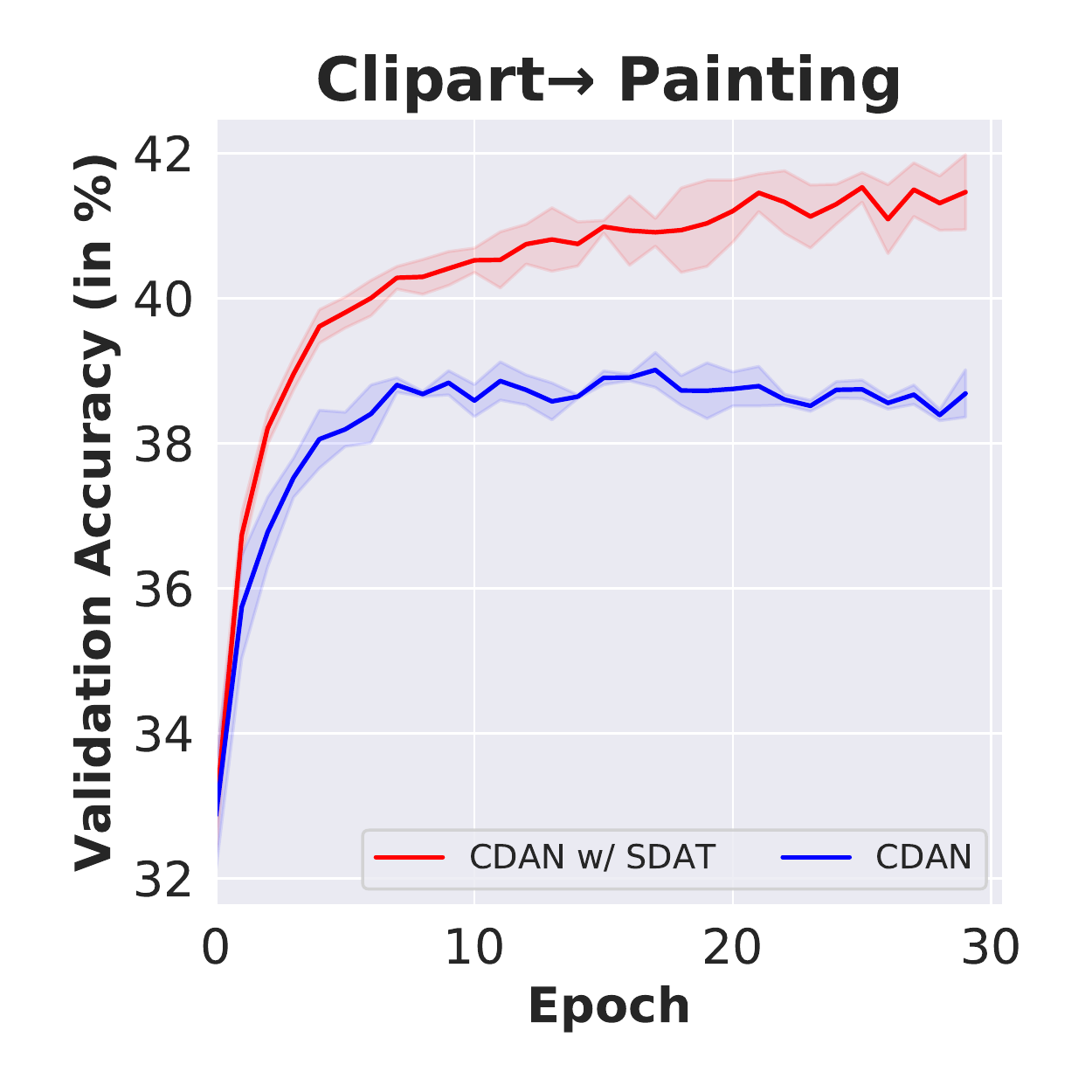}
  \label{fig:c2p}
\end{subfigure}
\begin{subfigure}[b]{0.3\linewidth}
  \centering
  \includegraphics[width=\textwidth]{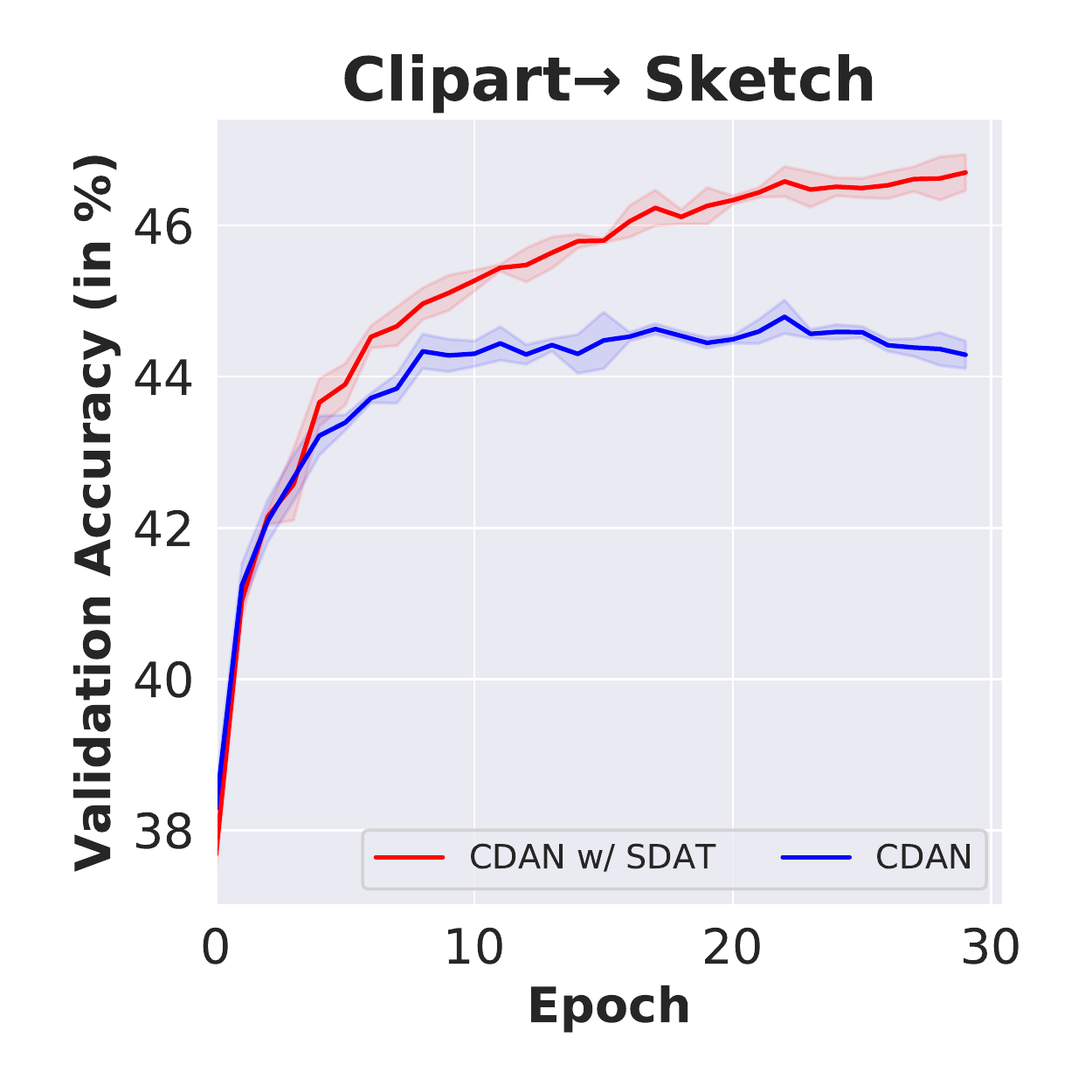}
  \label{fig:c2s}
\end{subfigure}
\begin{subfigure}[b]{0.3\linewidth}
  \centering
  \includegraphics[width=\textwidth]{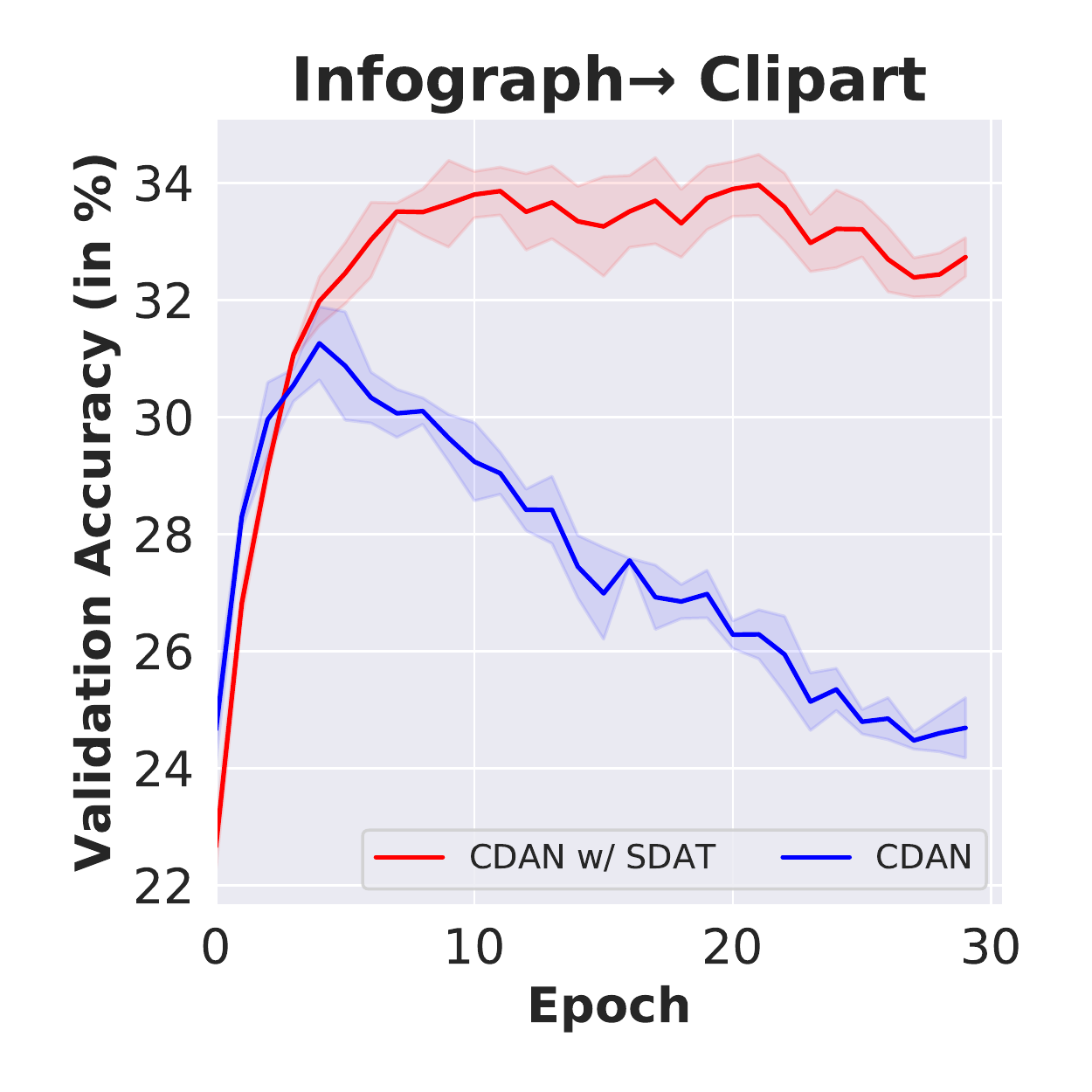}
  \label{fig:i2c}
\end{subfigure}\\
    \begin{subfigure}[b]{0.3\linewidth}
  \centering
  \includegraphics[width=\textwidth]{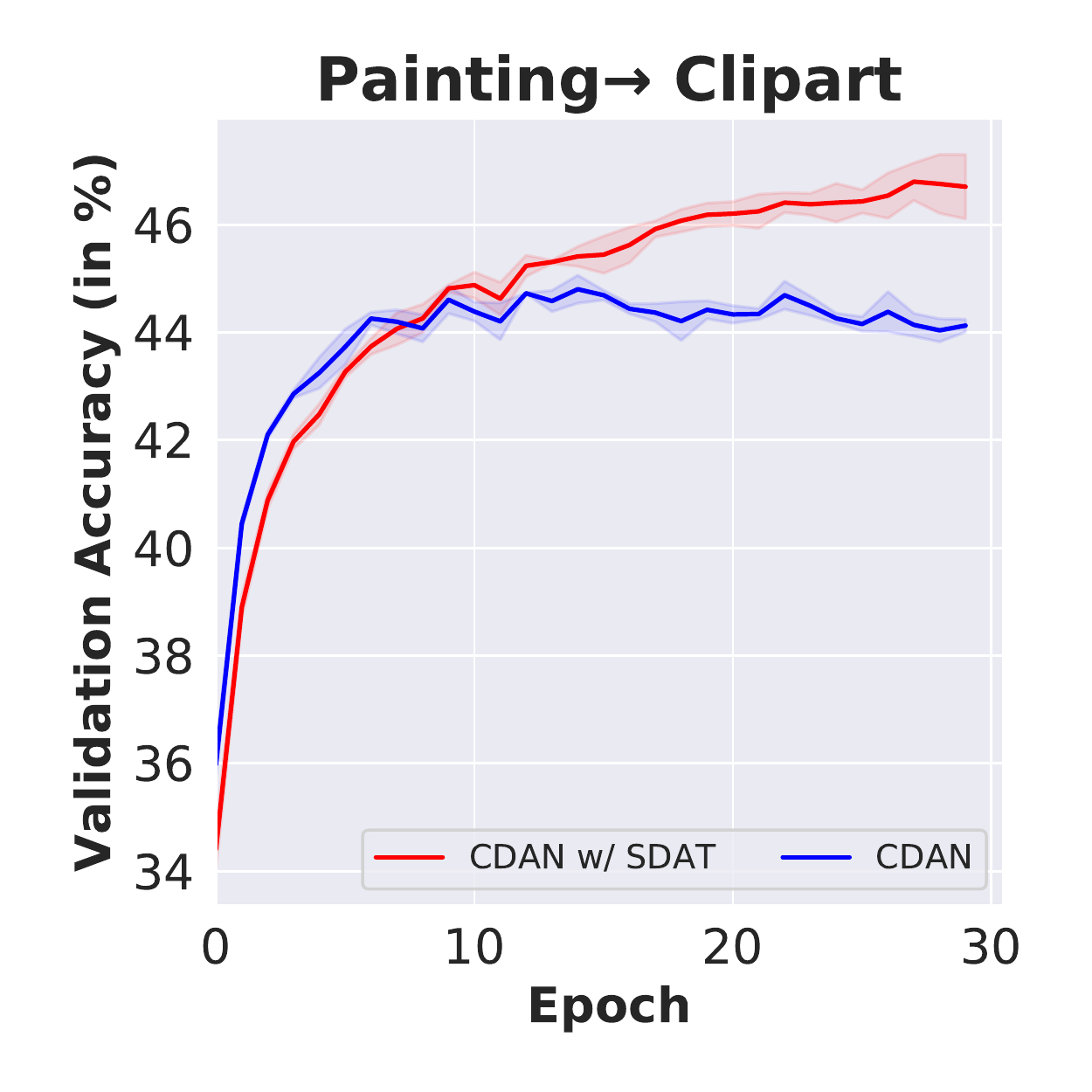}
  \label{fig:p2c}
\end{subfigure}
\begin{subfigure}[b]{0.3\linewidth}
  \centering
  \includegraphics[width=\textwidth]{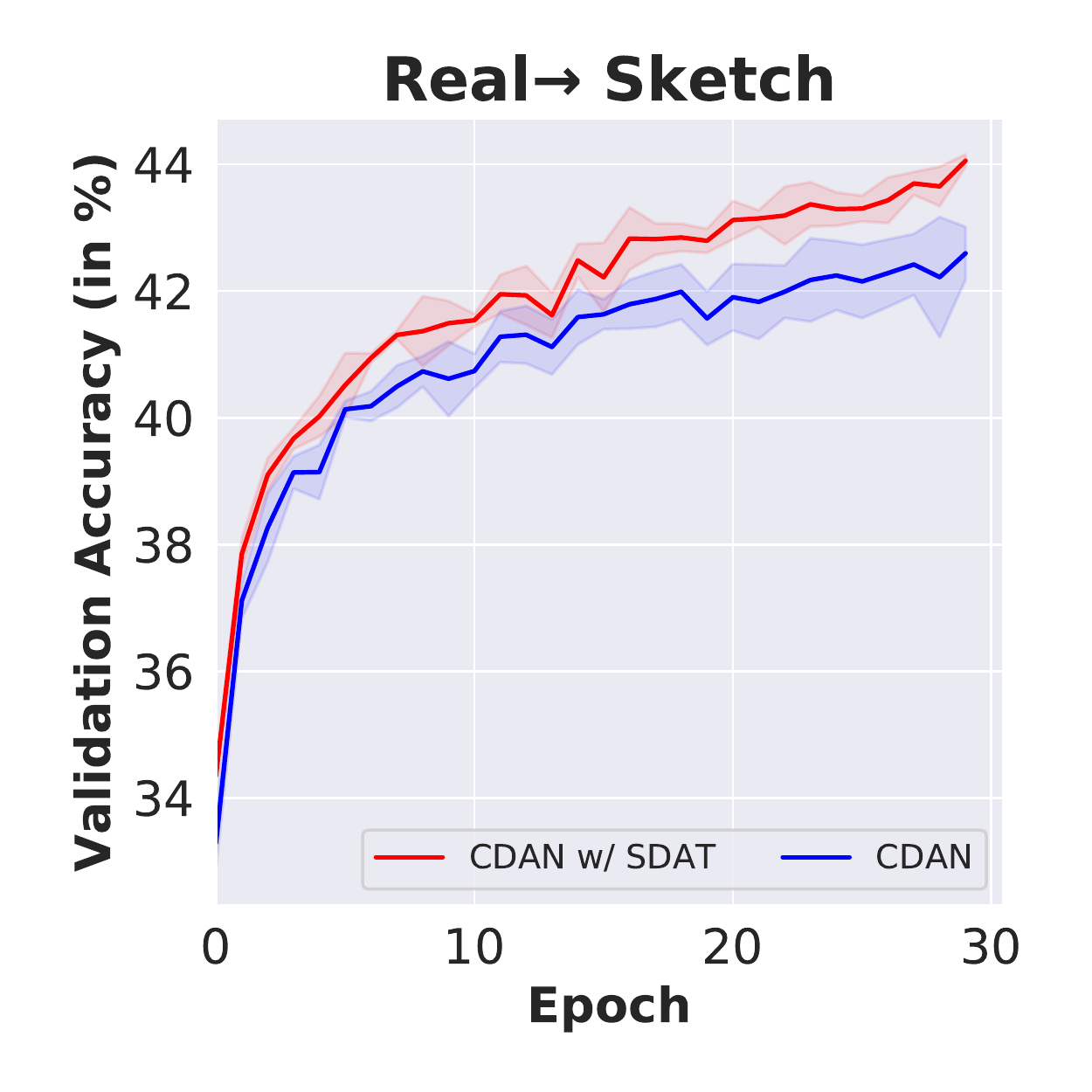}
  \label{fig:r2s}
\end{subfigure}
\begin{subfigure}[b]{0.3\linewidth}
  \centering
  \includegraphics[width=\textwidth]{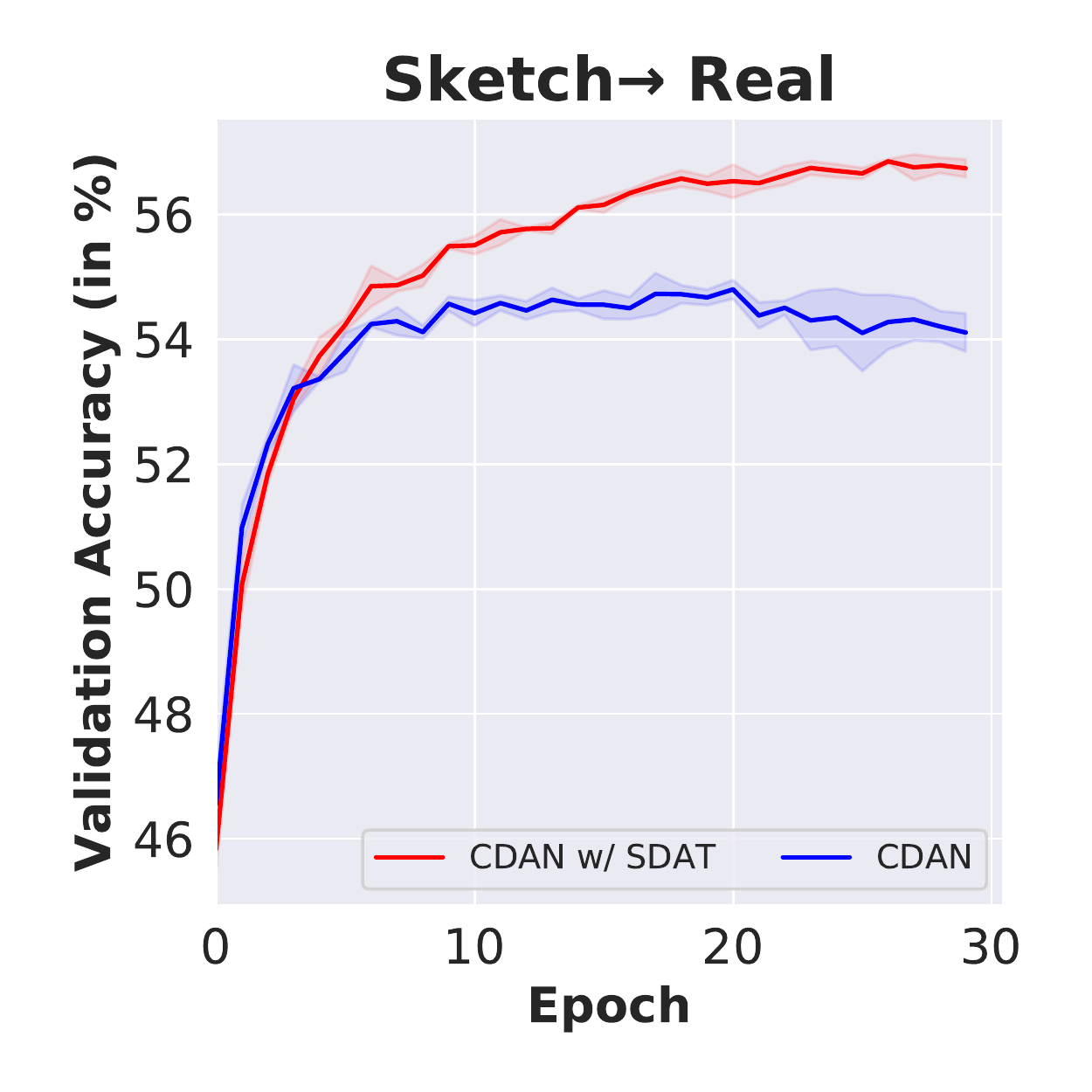}
  \label{fig:s2r}
\end{subfigure}
\caption{{Validation Accuracy across epochs on different splits of DomainNet. We run on three different random seeds and plot the error bar indicating standard deviation across runs. CDAN w/ SDAT consistently outperforms CDAN across different splits of DomainNet.}}
\label{fig:plotss}

\end{figure*}

\begin{table*}
    \centering
    \caption{ {DomainNet experiments over 3 different seeds (with ResNet backbone). We report the mean, standard deviation,  reported increase and average increase in the accuracy (in \%).}}
    \vskip 0.15in
     \begin{adjustbox}{max width=\columnwidth}
    \begin{tabular}{l|cc|cc}
    \hline
    {Split} &  CDAN & CDAN w/ SDAT & Reported Increase (Table \ref{table:domainnet}) & Average Increase\\
    \hline \hline
    \textbf{clp$\veryshortarrow$pnt} & 38.9 $\pm$  0.1 & 41.5 $\pm$ 0.3 & +2.6 & +2.6 \\
    \textbf{skt$\veryshortarrow$rel} & 55.1 $\pm$  0.2 & 57.1 $\pm$  0.1 & +2.2 & +2.0 \\
    \textbf{pnt$\veryshortarrow$clp } & 44.5 $\pm$  0.3 & 47.1 $\pm$  0.3 & +3.4 & +2.6\\
    \textbf{rel$\veryshortarrow$skt } & 42.4 $\pm$  0.4 & 43.9 $\pm$  0.1& +1.6 & +1.5\\
    \textbf{clp$\veryshortarrow$skt } & 44.9 $\pm$  0.2 & 47.3 $\pm$  0.1 & +2.3 & +2.4\\
    \textbf{inf$\veryshortarrow$clp} & 31.4 $\pm$  0.5 & 34.2 $\pm$  0.3 & +2.3 & +2.7\\
    \end{tabular}
    
    \label{tab:seed_exp}
    \end{adjustbox}
\end{table*}
\begin{table}[!t]
\centering
\captionsetup{width=\linewidth}
 \caption{{Median accuracy of last 5 epochs on DomainNet dataset with CDAN w/ SDAT. The number in the parenthesis indicates the increase in accuracy with respect to CDAN.}}
 \vskip 0.15in
\begin{adjustbox}{max width=\linewidth}
\begin{tabular}{c | c  c  c  c  c | c } 
 \hline
  \textbf{Target \textbf{($\rightarrow$)}} & \multirow{2}{*}{\textbf{clp}} & \multirow{2}{*}{\textbf{inf}} & \multirow{2}{*}{\textbf{pnt}} & \multirow{2}{*}{\textbf{real}} & \multirow{2}{*}{\textbf{skt}} & \multirow{2}{*}{\textbf{Avg}} \Tstrut\\  \textbf{Source ($\downarrow$)} &&&&&& \Bstrut\\
 \hline\hline
\multirow{2}{*}{\textbf{clp}} & - & 21.9  & 41.6 & 56.5 &  46.4 & 41.6 \Tstrut\\&& \textcolor{ForestGreen}{(+1.7)}& \textcolor{ForestGreen}{(+3.0)}&
 \textcolor{ForestGreen}{(+1.3)} & \textcolor{ForestGreen}{(+2.0)} & \textcolor{ForestGreen}{(+2.0)} \\ 

\multirow{2}{*}{\textbf{inf}} & 32.4 & - & 29.8 & 46.7 & 25.6 & 33.6 \Tstrut\\& \textcolor{ForestGreen}{(+7.9)}&&
  \textcolor{ForestGreen}{(+7.0)} & \textcolor{ForestGreen}{(+12.7)} & \textcolor{ForestGreen}{(+5.4)} & \textcolor{ForestGreen}{(+8.2)} \\

\multirow{2}{*}{\textbf{pnt}} &  47.2 & 21.0 & - & 57.6 & 41.5 & 41.8 \Tstrut\\ & \textcolor{ForestGreen}{(+2.9)}& 
\textcolor{ForestGreen}{(+1.1)} 
 &&
 \textcolor{ForestGreen}{(+1.0)} &
\textcolor{ForestGreen}{(+2.4)} &
 \textcolor{ForestGreen}{(+1.8)} \\

\multirow{2}{*}{\textbf{real}} &  56.5 & 25.5 & 53.9 & - & 43.5 & 44.8 \Tstrut \\ & \textcolor{ForestGreen}{(+0.7)} &
 \textcolor{ForestGreen}{(+0.9)} &
 \textcolor{ForestGreen}{(+0.5)} &
 &
 \textcolor{ForestGreen}{(+1.3)} &
 \textcolor{ForestGreen}{(+0.8)} \\

\multirow{2}{*}{\textbf{skt}} &  59.1 & 22.1 & 48.2 & 56.6 & - & 46.5 \Tstrut \\ &  \textcolor{ForestGreen}{(+3.0)} & 
 \textcolor{ForestGreen}{(+1.7)} &
 \textcolor{ForestGreen}{(+3.1)} &
 \textcolor{ForestGreen}{(+2.9)} &
&
 \textcolor{ForestGreen}{(+2.7)} 
\\\hline

\multirow{2}{*}{\textbf{Avg}} &  48.8 & 22.6 & 43.4 & 54.3 & 39.2 & 41.7 \Tstrut \\ & \textcolor{ForestGreen}{(+3.6)} &
 \textcolor{ForestGreen}{(+1.3)} &
 \textcolor{ForestGreen}{(+3.4)} &  \textcolor{ForestGreen}{(+4.5)} &
 \textcolor{ForestGreen}{(+2.8)} &
 \textcolor{ForestGreen}{(+3.1)}

\end{tabular}
\label{table:domainnet_median}
 \end{adjustbox}
 \end{table}

As the Office-Home dataset is smaller (i.e., 44 images per class) in comparison to DomainNet we find that there exists some variance in baseline CDAN results (This is also reported in the well-known benchmark for DA \citep{dalib}). For establishing the empirical soundness, we report results of 4 different dataset splits on 3 seeds. It can be seen in Table \ref{tab:seed_exp_office} that even though there is variance in baseline results, our combination of CDAN w/ SDAT can produce consistent improvement across different random seeds. This further establishes the empirical soundness of our procedure.

\begin{table*}[hbt!]
    \centering
    \caption{ {Office-Home experiments over 3 different seeds (with ResNet-50 backbone). We report the mean, standard deviation,  reported increase and average increase in the accuracy (in \%).}}
    \vskip 0.15in
     \begin{adjustbox}{max width=\columnwidth}
    \begin{tabular}{l|cc|cc}
    \hline
    {Split} &  CDAN & CDAN w/ SDAT & Reported Increase (Table \ref{tab:officehome}) & Average Increase\\
    \hline \hline
    \textbf{Ar$\veryshortarrow$Cl} & 53.9 $\pm$  0.2 & 55.5 $\pm$ 0.2 & +1.7 & +1.6 \\
    \textbf{Ar$\veryshortarrow$Pr} & 70.6 $\pm$  0.4 & 72.1 $\pm$  0.4 & +1.6 & +1.5 \\
    \textbf{Rw$\veryshortarrow$Cl} & 60.7 $\pm$  0.5 & 61.8 $\pm$  0.4 & +1.3 & +1.1\\
    \textbf{Pr$\veryshortarrow$Cl} & 54.7 $\pm$  0.4 & 55.5 $\pm$  0.4& +0.6 & +0.8\\
    \end{tabular}
    
    \label{tab:seed_exp_office}
    \end{adjustbox}
\end{table*}

\section{PyTorch Pseudocode for SDAT}
In the code snippet below, we show that with a few changes in the code, SDAT can be easily integrated with any DAT algorithm. SDAT requires an additional forward pass and gradient computation, as shown below.
\label{app:pytorchcode}
\begin{lstlisting}[language=Python,basicstyle=\ttfamily\scriptsize,escapeinside={(*}{*)}]
# task_loss_fn refers to the function to calculate task loss. 
# (For classification settings, this can be Cross Entropy Loss).

# optimizer refers to the smooth optimizer which contains parameters of the feature extractor and classifier.
optimizer.zero_grad()
# ad_optimizer refers to standard SGD optimizer which contains parameters of domain classifier.
ad_optimizer.zero_grad()

# Calculate task loss
class_prediction, feature = model(x)
task_loss = task_loss_fn(class_prediction, label)
task_loss.backward()

# Calculate (* $\hat{\epsilon}$ (w) *) and add it to the weights
optimizer.first_step()

# Calculate task loss and domain loss
class_prediction, feature = model(x)
task_loss = task_loss_fn(class_prediction, label)
domain_loss = domain_classifier(feature)
loss = task_loss + domain_loss
loss.backward()

# Update parameters (Sharpness-Aware update)
optimizer.second_step()
# Update parameters of domain classifier
ad_optimizer.step()
\end{lstlisting}

\end{document}